\documentclass[sigconf]{acmart}

\usepackage[utf8]{inputenc} 
\usepackage[T1]{fontenc}    
\usepackage{hyperref}       
\usepackage{url}            
\usepackage{booktabs}       
\usepackage{amsfonts}       
\usepackage{nicefrac}       
\usepackage{microtype}      
\usepackage{xcolor}         
\usepackage{graphicx}
\usepackage{amsmath}
\usepackage{amsthm}
\usepackage{booktabs}
\usepackage{algorithm}
\usepackage{algpseudocode}
\usepackage{multirow}
\usepackage{subfigure}
\usepackage{listings}
\usepackage{threeparttable}
\usepackage{wrapfig}
\usepackage{diagbox}
\usepackage{pifont}

\newtheorem{lemma}{Lemma}

\newtheorem{proposition}{Proposition}

\newtheorem{remark}{Remark}

\newcommand{\cmark}{\ding{51}}%
\newcommand{\xmark}{\ding{55}}%

\definecolor{pink}{rgb}{0.858, 0.188, 0.478}
\definecolor{commentcolor}{RGB}{110,154,155}   
\newcommand{\PyComment}[1]{\ttfamily\textcolor{commentcolor}{\# #1}}  
\newcommand{\PyCode}[1]{\ttfamily\textcolor{black}{#1}} 

\newcommand{\nosection}[1]{\vspace{2pt}\noindent\textbf{#1\ }}
\newcommand{\code}[0]{\url{https://github.com/EdisonLeeeee/MaskGAE}}

\AtBeginDocument{%
  \providecommand\BibTeX{{%
    \normalfont B\kern-0.5em{\scshape i\kern-0.25em b}\kern-0.8em\TeX}}}

\copyrightyear{2023}
\acmYear{2023}
\setcopyright{acmlicensed}
\acmConference[KDD '23] {Proceedings of the 29th ACM SIGKDD Conference on Knowledge Discovery and Data Mining}{August 6--10, 2023}{Long Beach, CA, USA.}
\acmBooktitle{Proceedings of the 29th ACM SIGKDD Conference on Knowledge Discovery and Data Mining (KDD '23), August 6--10, 2023, Long Beach, CA, USA}
\acmPrice{15.00}
\acmISBN{979-8-4007-0103-0/23/08}
\acmDOI{10.1145/3580305.3599546}

\begin{document}

\title{What's Behind the Mask: Understanding Masked Graph Modeling for Graph Autoencoders}

\author{Jintang Li}
\authornote{Both authors contributed equally to this research.}
\affiliation{\institution{Sun Yat-sen University}
    \country{}}
\email{lijt55@mail2.sysu.edu.cn}

\author{Ruofan Wu}
\authornotemark[1]
\affiliation{\institution{Ant Group}
    \country{}}
\email{ruofan.wrf@antgroup.com}

\author{Wangbin Sun}
\affiliation{\institution{Sun Yat-sen University}
    \country{}}
\email{sunwb7@mail2.sysu.edu.cn}

\author{Liang Chen}
\affiliation{\institution{Sun Yat-sen University}
    \country{}}
\email{chenliang6@mail.sysu.edu.cn}

\author{Sheng Tian}
\affiliation{\institution{Ant Group}
    \country{}}
\email{tiansheng.ts@antgroup.com}

\author{Liang Zhu}
\affiliation{\institution{Ant Group}
    \country{}}
\email{tailiang.zl@antgroup.com}

\author{Changhua Meng}
\affiliation{\institution{Ant Group}
    \country{}}
\email{changhua.mch@antgroup.com}

\author{Zibin Zheng}
\affiliation{\institution{Sun Yat-sen University}
    \country{}}
\email{zhzibin@mail.sysu.edu.cn}

\author{Weiqiang Wang}
\affiliation{\institution{Ant Group}
    \country{}}
\email{weiqiang.wwq@antgroup.com}

\renewcommand{\shortauthors}{Li and Wu, et al.}

\begin{abstract}
    The last years have witnessed the emergence of a promising self-supervised learning strategy, referred to as masked autoencoding. However, there is a lack of theoretical understanding of how masking matters on graph autoencoders (GAEs). In this work, we present masked graph autoencoder (MaskGAE), a self-supervised learning framework for graph-structured data. Different from standard GAEs, MaskGAE adopts masked graph modeling (MGM) as a principled pretext task - masking a portion of edges and attempting to reconstruct the missing part with partially visible, unmasked graph structure. To understand whether MGM can help GAEs learn better representations, we provide both theoretical and empirical evidence to comprehensively justify the benefits of this pretext task. Theoretically, we establish close connections between GAEs and contrastive learning, showing that MGM significantly improves the self-supervised learning scheme of GAEs. Empirically, we conduct extensive experiments on a variety of graph benchmarks, demonstrating the superiority of MaskGAE over several state-of-the-arts on both link prediction and node classification tasks.\footnote{Code is made publicly available at \code}
\end{abstract}



\maketitle

\section{Introduction}

Self-supervised learning, which learns broadly useful representations from unlabeled data in a task-agnostic way, has emerged as a popular and empirically successful learning paradigm for graph neural networks (GNNs)~\cite{wu2021self}. The key insight behind self-supervised learning is to obtain supervisory signals from the data itself with different handcrafted auxiliary tasks (so-called \emph{pretext tasks}). The past few years have witnessed the success of graph self-supervised learning in a wide range of graph-related fields~\cite{velickovic2019deep,chen2021understanding,STEP}, particularly chemical and biomedical science~\cite{rong2020self,DBLP:conf/ijcai/ZhaoLHLZ21}, where label annotation can be very costly or even impossible to acquire.

For years, there are several research efforts in an attempt to exploit vast unlabeled data for self-supervised learning. Among contemporary approaches, contrast learning~\cite{velickovic2019deep,graphcl,grace} is one of the most widespread self-supervised learning paradigms on graph data. It learns representations that are invariant to different augmentation views of graphs, achieving remarkable success in various graph representation learning tasks~\cite{wu2021self}. Despite being effective and prevalent, graph contrastive methods highly rely on specialized and complex pretext tasks for self-supervised learning~\cite{mvgrl}, with data augmentation being crucial for contrasting different structural views of graphs~\cite{YouCSW21,lee2021augmentation}.

There is another prominent line of research attempts to learn representations through generative perspectives, with graph autoencoders as representative examples. Graph autoencoders (GAEs) are a family of self-supervised learning models that take the graph input itself as self-supervision and learn to reconstruct the graph structure~\cite{kipf2016variational,PanHLJYZ18_arge}.
Compared to contrastive methods, GAEs are generally very simple to implement and easy to combine with existing frameworks, since they naturally leverage graph reconstruction as pretext tasks without need of augmentations for view generations~\cite{wu2021self}.
However, literature has shown that GAEs following such a simple graph-reconstruction principle might over-emphasize proximity information that is not always beneficial for self-supervised learning~\cite{velickovic2019deep,mvgrl,grace}, making it less applicable to other challenging tasks beyond link prediction. Therefore, there is a need for designing better pretext tasks for GAEs.

\begin{figure}
    \centering
    \includegraphics[width=\linewidth]{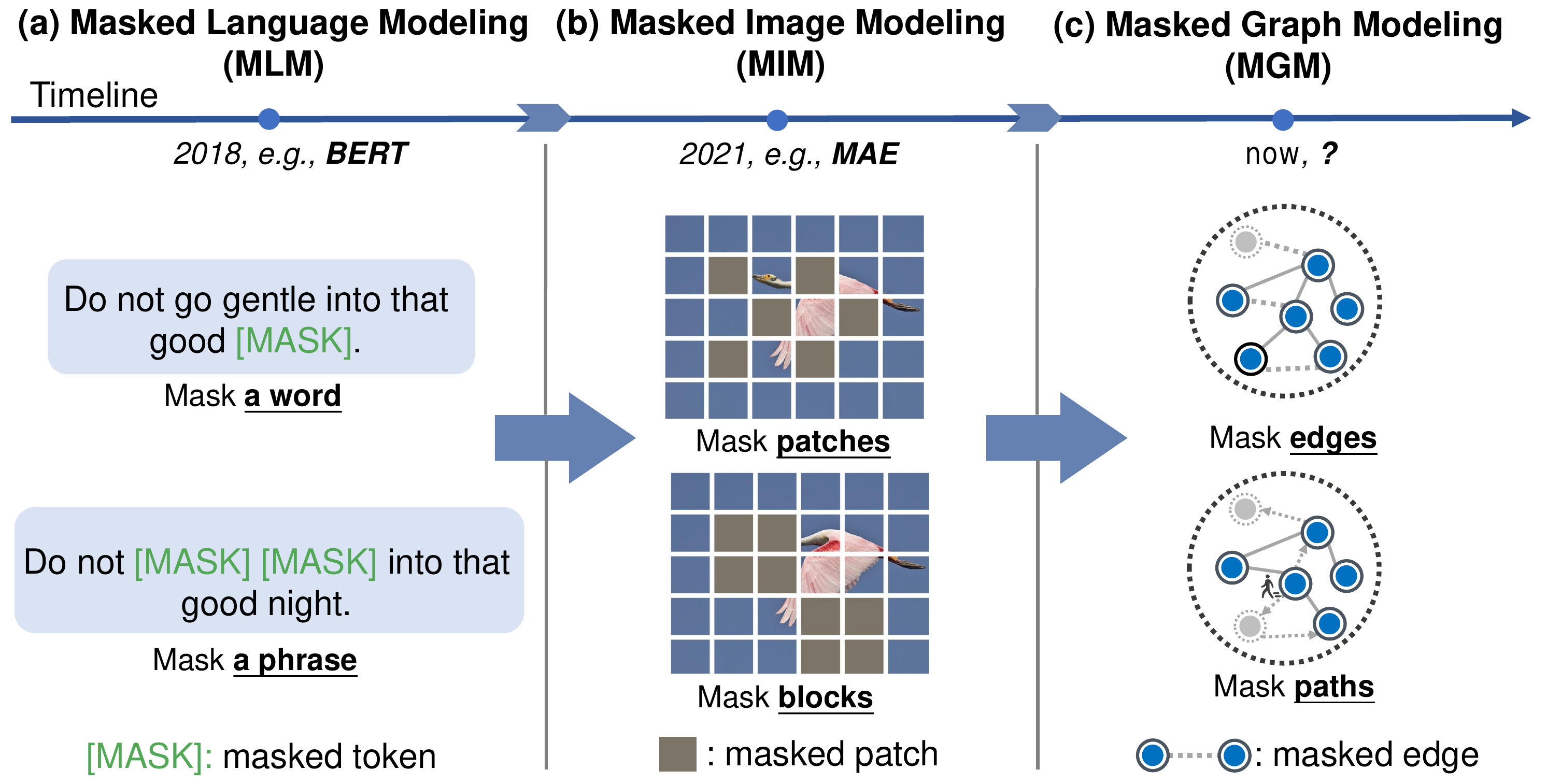}
    \caption{\textbf{From left to right}: illustrative examples of \textbf{(a)} masked language modeling (MLM), \textbf{(b)} masked image modeling (MIM), and \textbf{(c)} masked graph modeling (MGM) paradigms with different masking strategies. Similar to MLM and MIM, the goal of MGM is to learn representations by predicting randomly masked edges based on visible structure.}
    \label{fig:comparison}
    \vspace{-5mm}
\end{figure}

Masked autoencoding, whose goal is to reconstruct masked signals from unmasked input under the autoencoder framework, has recently advanced the state-of-the-art and provided valuable insights in both language and vision research. As shown in Figure~\hyperref[fig:comparison]{1(a)} and~\hyperref[fig:comparison]{1(b)},
masked language modeling (MLM) and masked image modeling (MIM) have been widely applied to text and image data, with prominent examples including BERT~\cite{bert} and MAE~\cite{mae}, respectively. Despite the fact that masked autoencoding has shown promise in benefiting visual and language representation learning, the applicability on graph data has yet to be substantially investigated in this discipline. Actually, masked autoencoding should also be a good fit for graph data, since each edge can easily be masked or unmasked as self-supervisions. In light of this, a natural question arises: \emph{whether masked autoencoding, or masked graph modeling, would advance GAEs in self-supervised learning as well?}

\nosection{Present work.} In this work, we seek to continue the success of MLM and MIM by introducing masked graph modeling (MGM) as a principled pretext task for graph-structured data.
As shown in Figure~\hyperref[fig:comparison]{1(c)}, the core idea behind MGM is to remove a portion of the input graph and learn to predict the removed content such as edges. Following this philosophy, we propose masked graph autoencoder (MaskGAE), a self-supervised learning framework that leverages the idea of masking and predicting through the node and edge-level reconstruction.
Our framework is theoretically grounded by explicitly relating GAEs to contrastive learning and demonstrating the benefits of MGM in improving the self-supervised learning scheme of GAEs. Specifically, we reveal that the learning objective of GAEs is equivalent to contrastive learning, in which paired subgraphs naturally form two structural views for contrasting. Most importantly, masking on an edge can reduce redundancy for two contrastive subgraph views in GAEs and thus benefit contrastive learning.

This paper offers the following main contributions:
\begin{itemize}
    \item MaskGAE, a simple yet effective self-supervised learning framework for graphs.
    \item A comprehensive theoretical analysis of the proposed framework along with guidance on the design of pretext training tasks (i.e., MGM) for GAEs.
    \item A new form of structured masking strategy to facilitate the MGM task, where edges in a contiguous region are masked together.
    \item An in-depth experimental evaluation of the proposed framework, demonstrating the effectiveness of MaskGAE on both link prediction and node classification tasks.
\end{itemize}

We believe that our work is a step forward in the development of simple and provably powerful graph self-supervised learning frameworks, and hope that it would inspire both theoretical and practical future research.

\section{Related Work}

\paragraph{Graph contrastive methods.}
Contrastive methods follow the principle of mutual information maximization~\cite{infomax}, which typically works to maximize the correspondence between the representations of an instance (e.g., node, subgraph, or graph) in its different augmentation views. Essentially, good data augmentations and pretext tasks are the prerequisites for contrastive learning. There is a vast majority of research on graph contrastive learning~\cite{velickovic2019deep,graphcl,mvgrl}. To our best knowledge, graph contrastive methods are currently the most successful approaches to learning useful and expressive representations in a self-supervised fashion. However, one limitation shared by all these successful approaches is that they highly rely on the design of pretext tasks and augmentations techniques (usually summarized from many trial-and-error) to provide useful self-supervision for learning better representations~\cite{YouCSW21}.

\paragraph{Graph autoencoders.}
GAEs in the form proposed by~\cite{kipf2016variational} are a family of generative models that map (encode) nodes to low-dimensional representations and reconstruct (decode) the graph. Following the autoencoding philosophy, latest approaches have demonstrated their efficacy in modeling node relationships and learning robust representations from a graph~\cite{PanHLJYZ18_arge}. Although GAEs progressed earlier than graph contrastive approaches, they remained out of the mainstream for a long time. This is mainly due to the known limitations that GAEs suffer from.
As revealed in literature~\cite{velickovic2019deep,mvgrl}, GAEs tend to over-emphasize proximity information at the expense of structural information, leading to relatively poor performance on downstream tasks beyond link prediction. Despite the capability of GAEs being largely limited by the pretext task, there has been little attention paid to a better design of pretext tasks for improving the self-supervised learning scheme of GAEs.

\paragraph{Masked autoencoding.}
Masked autoencoding is one such learning task: masking a portion of input signals and attempting to predict the contents that are hidden by the mask~\cite{mae}. Masked language modeling (MLM)~\cite{bert} is the first successful application of masked autoencoding in natural language processing. Typically, MLM is a fill-in-the-blank self-supervised learning task, where a model learns representations by predicting what a masked word should be with the context words surrounding the token.
Recently, masked image modeling (MIM)~\cite{mae,simmim} follows a similar principle to learn representations by predicting the missing parts at the pixel or patch level. MIM is gaining renewed interest from both industries and academia, leading to new state-of-the-art performance on broad downstream tasks. In a recent exploration \cite{zhang2022mask}, the authors discovered an implicit connection between MIM and contrastive learning.
Despite the popularity in language and vision research, the techniques of masked autoencoding are relatively less explored in the graph domain.
Until very recently, there were a few trials attempting to bridge this gap. In work concurrent with the present paper, MGAE~\cite{mgae} and GraphMAE~\cite{graphmae} seek to apply this idea directly to graph data as a self-supervised learning paradigm, by performing masking strategies on graph structure and node attributes, respectively. However, they present only empirical results with experimental trial-and-error, which lacks further theoretical justification for a better understanding of the potential benefits of masked graph modeling.

\section{Problem Formulation and Preliminaries}
\nosection{Problem formulation.}
Let $\mathcal{G}=(\mathcal{V}, \mathcal{E})$ be an undirected and unweighted graph, where $\mathcal{V}=\{v_i\}$ is a set of nodes and $\mathcal{E} \subseteq \mathcal{V} \times \mathcal{V}$ is the corresponding edges.
Optionally, each node $v\in\mathcal{V}$ is associated with a $d$-dimensional feature vector $x_v \in \mathbb{R}^d$.
The goal of most graph self-supervised learning methods, including ours, is to learn a graph encoder $f_\theta$, which maps between the space of graph $\mathcal{G}$ and their low-dimensional latent representations $\mathbf{Z}=\{z_i\}^{|\mathcal{V}|}_{i=1}$, such that $f_\theta (\mathcal{G})=\mathbf{Z} \in \mathbb{R}^{|\mathcal{V}| \times d_h}$ best describes each node in $\mathcal{G}$, where $d_h$ is the embedding dimension.

\nosection{Masked graph modeling (MGM).}
Following the masked autoencoding principle, we introduce MGM as a pretext task for graph self-supervised learning in this paper. Similar to MLM and MIM tasks in language and vision research, MGM aims to assist the model to learn more useful, transferable, and generalized representations from unlabeled graph data through masking and predicting. This self-supervised pre-training strategy is particularly scalable when applied to GAEs since only the unmasked graph structure is processed by the network.

\section{Theoretical Justification and Motivation}
\label{sec:theory}
\subsection{Revisiting graph autoencoders}
GAEs, which leverage naturally occurring pairs of similar and dissimilar nodes in a graph as self-supervised signals, have shown the advantage of learning graph structures and node representations. GAEs adopt the classic encoder-decoder framework, which aims to decode from the low-dimensional representations that encode the graph by optimizing the following binary cross-entropy loss:
\begin{equation}
    \label{eq:gae}
    \begin{aligned}
         & \mathcal{L}^+ = \frac{1}{|\mathcal{E}^+|}\sum_{(u, v)\in \mathcal{E}^+}\log h_\omega(z_u, z_v),                                     \\
         & \mathcal{L}^- =  \frac{1}{|\mathcal{E}^-|}\sum_{(u^\prime, v^\prime)\in \mathcal{E}^-}\log(1-h_\omega(z_{u^\prime}, z_{v^\prime})), \\
         & \mathcal{L}_\text{GAEs} = - \left(\mathcal{L}^++\mathcal{L}^-\right)
    \end{aligned}
\end{equation}
where $z$ is the node representation obtained from an encoder $f_\theta$ (e.g., a GNN); $\mathcal{E}^+$ is a set of positive edges while $\mathcal{E}^-$ is a set of negative edges sampled from graph; Typically, $\mathcal{E}^+=\mathcal{E}$. We denote $h_\omega$ a decoder with parameters $\omega$.

\subsection{Connecting GAEs to contrastive learning}
We now provide some intuition that connects GAEs to contrastive learning from the viewpoint of information theory.
Our analysis will be based on the \textbf{homophily} assumption, i.e., the underlying semantics of nodes $u$ and $v$ are more likely to be the same if they are connected by an edge. We adopt the \emph{information-maximization (infomax)} viewpoint of contrastive learning \cite{tian2020what, tsai2021selfsupervised}. Specifically, let $I(X; Y)$ be the \emph{mutual information (MI)} between random variables $X$ and $Y$ taking values in $\mathcal{X}$ and $\mathcal{Y}$, respectively. An important alternative characterization of MI is the following Donsker-Varadhan variational representation \cite{polyanskiy2014lecture}:
\begin{equation}\label{eq:dv}
    \begin{aligned}
         & I(X; Y)              = \sup_{c: \mathcal{X} \times \mathcal{Y} \mapsto \mathbb{R}}\mathcal{I}_c(X; Y),                  \\
         & \mathcal{I}_c(X; Y)  = \mathbb{E}_{x, y \sim P_{XY}} c(x, y) - \log \mathbb{E}_{x, y \sim P_X \times P_Y}(e^{c(x, y)}),
    \end{aligned}
\end{equation}
with the \emph{critic} function $c$ \cite{pmlr-v97-poole19a} ranging over the set of integrable functions taking two arguments. Under the context of GAEs, we identify $X$ and $Y$ with the corresponding $k$-hop subgraphs $\mathcal{G}^k(u)$ and $\mathcal{G}^k(v)$ of \textbf{adjacent} nodes $u$ and $v$, with randomness over the generating distribution of the graph as well as the generating distributions of node features.\footnote{The negative sampling process in the objective \eqref{eq:gae} is biased if negative pairs are sampled from disconnected node pairs. Since it does not affect our analysis, we defer the discussion to Appendix~\ref{appendix:discussions}.}
Denote the corresponding joint and marginal distributions as $P_{UV}$, $P_U$, and $P_V$, respectively. Under this formulation, we may view Eq.\eqref{eq:gae} as an empirical approximation of the following population-based objective:
\begin{equation}\label{eq:gae_pop}
    \begin{aligned}
         & \mathcal{I}^+_h(U;V) = \mathbb{E}_{u, v \sim P_{UV}}\log h(u, v),                              \\
         & \mathcal{I}^-_h(U;V) = \mathbb{E}_{u^\prime \sim P_U, v^\prime \sim P_V}\log(1-h(u, v)),       \\
         & \mathcal{I}^{\text{GAEs}}_h(U;V) = - \left( \mathcal{I}^+_h(U;V)+\mathcal{I}^-_h(U;V) \right).
    \end{aligned}
\end{equation}
With slight abuse of notation, we denote $I(U; V)$ as the MI between $\mathcal{G}^k(U)$ and $\mathcal{G}^k(V)$. The following lemma establishes the connection between GAE and contrastive learning over graphs:
\begin{lemma}\label{lem:equiv}
    Let $h^* \in \arg\min_{h \in \mathcal{H}} \mathcal{I}^{\text{GAEs}}_h(U;V)$. Then we have $\mathcal{I}_{h^*}(U; V) = I(U; V)$.
\end{lemma}
Here the domain $\mathcal{H}$ stands for the product space of all possible pairs of subgraphs. The above lemma is a direct consequence of the results in \cite{pmlr-v97-poole19a, tsai2020neural}. Lemma \ref{lem:equiv} states that minimizing the GAE objective \eqref{eq:gae} is in population equivalent to maximizing the mutual information between the $k$-hop subgraphs of adjacent nodes (here $k$ depends on the receptive fields of the encoder). Now suppose that the parameterization $h_\omega(z_u, z_v)$ is sufficiently expressive for approximating $h(u, v)$ for any $h \in \mathcal{H}$, it follows from standard results in M-estimation theory \cite{van2000asymptotic} that the corresponding empirical minimizer of $ \mathcal{I}^{\text{GAEs}}$ converges to the maximizer of $\mathcal{I}$ in probability.
Further discussions on the expressivity and approximation issues can be found in Appendix~\ref{appendix:discussions}.
\begin{remark}
    In a contemporary work~\cite{zhang2022mask}, the authors also made a ``contrastive interpretation'' of the masking procedure in MIM as conducting contrastive learning in an \textbf{implicit form}, in which a contrastive-type loss is used to lower bound the reconstruction loss in the MIM setting. Our formulation under the MGM setting suggests a more \textbf{direct} connection to contrastive learning (in terms of asymptotic equivalent solutions) due to the specific form of the GAE objective~\eqref{eq:gae}.
\end{remark}

\begin{figure*}[t]
    \centering
    \includegraphics[width=\linewidth]{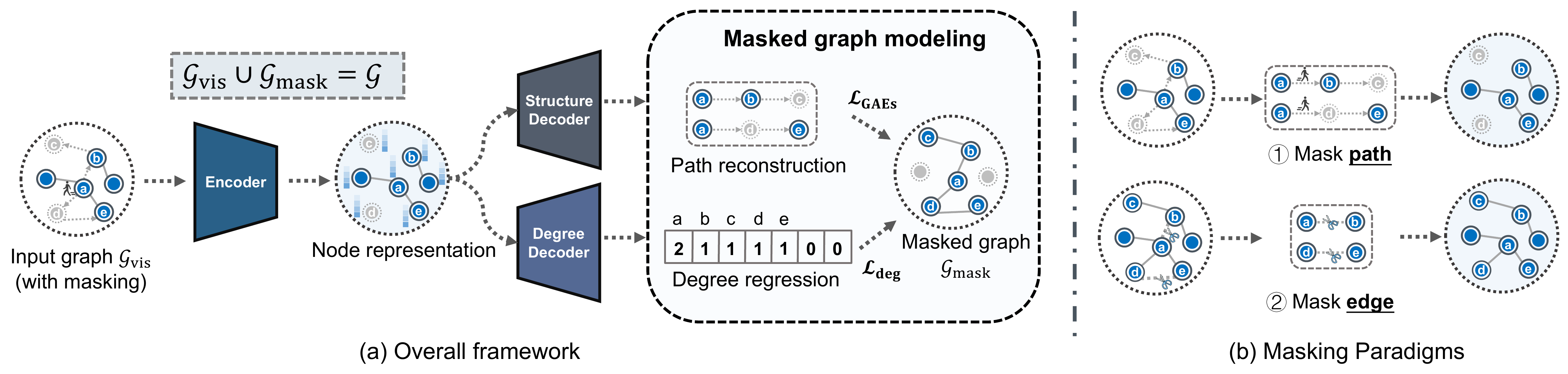}
    \caption{Overview of proposed MaskGAE framework, performing masked graph modeling with an asymmetric encoder-decoder design. During the self-supervised learning phase, an input graph $\mathcal{G}_{\text{vis}}$ is provided, but some of the paths (a set of adjacent edges) are masked. The goal of the model is then to learn to predict the existence of masked edges in the graph $\mathcal{G}_{\text{mask}}$ and the degree of the associated nodes, respectively.
    }
    \label{fig:framework}
\end{figure*}

\subsection{Task related information \& redundancy of GAEs}
\label{sec:redundancy}
The (asymptotic) equivalence of learned representations of GAEs and contrastive learning does not necessarily imply good performance regarding \emph{downstream tasks}. Recent progress on information-theoretic viewpoints of contrastive learning \cite{tian2020what, tsai2021selfsupervised} suggest that for contrastive pretraining to succeed in downstream tasks, the \emph{task irrelevant information} shall be reasonably controlled. Formally, let $U, V$ be random variables of the two contrasting views, and $T$ denote the target of the downstream task.\footnote{Technically, $T$ could be further relaxed to be any \emph{sufficient statistic} of the underlying task, we adopt the more direct formulation for representation clarity} Denote $I(U;V|T)$ as the conditional mutual information of $U$ and $V$ given $T$, we have the following simple identity which is a direct consequence of the chain rule \cite{polyanskiy2014lecture}.
\begin{equation} \label{eq:info_relation}
    \underbrace{I(U; T)}_{\text{supervised goal}} = \underbrace{I(U; V)}_{\text{self-supervised goal}} + \underbrace{I(U; T|V)}_{\text{task relevance}} - \underbrace{I(U; V|T)}_{\text{task irrelevanc}}.
\end{equation}
The relation \eqref{eq:info_relation} implies that, for successful application of contrastively-pretrained representations to downstream tasks, we need both $I(U; T|V)$ and $I(U; V|T)$ to be small.\footnote{When both $I(U; T|V)$ and $I(U; V|T)$ are large, the equation becomes meaningless in that $T$ and $V$ should be nearly independent}
A small value of $I(U; T|V)$ is a standard assumption in information-theoretic characterizations of self-supervised learning \cite{sridharan2008information, tsai2021selfsupervised}. The term $I(U; V|T)$ measures the \emph{task-irrelevant} information contained in the two contrastive views $U$ and $V$, regarding the downstream task $T$, and cannot be reduced via algorithmic designs of encoders and decoders \cite{tsai2021selfsupervised}. It is therefore of interest to examine the efficacy of GAEs via assessing the task-irrelevant information, or redundancy, of the two contrasting views that use $k$-hop subgraphs of adjacent nodes.\par
Intuitively, for certain kinds of downstream task information $T$, we might expect $I(U;V|T)$ to be large under the GAE formulation, since $k$-hop subgraphs of two adjacent nodes share a (potentially) large common subgraph. From a computational point of view, the phenomenon of overlapping subgraphs may affect up to $k-1$ layers of GNN message passing and aggregation during the encoding stage of both nodes, thereby creating a large correlation between the representations, even when the encoder has little relevance with the downstream task. To further justify the above reasoning, we give a lower bound of $I(U;V|T)$ under an independence assumption between graph topology and node features:
\begin{proposition}\label{prop:lower_bound}
    Suppose that the task-related information is completely reflected by the \emph{topological structure} of the underlying subgraph, i.e., $T$ is all the topological information of the underlying graph. Moreover suppose the node features are generated independently from the graph topology, with the feature of each node being(coordinate-wise) independently sampled from a zero-mean distribution supported in $[-1, 1]$ with variance $\gamma$. Let $N_k$ be an upper bound for the size of $k$-hop subgraph associated with each node $v$, i.e., $\max_{v \in \mathcal{V}}|\mathcal{G}^k(v)|_0 \le N_k$, we have the following lower bound of $I(U; V|T)$:
    \begin{align}
        I(U; V| T) \ge \frac{\left(\mathbb{E}[N^k_{uv}]\right)^2}{N_k} \gamma^2,
    \end{align}
    where $N^k_{uv}$ is the size of the overlapping subgraph of $\mathcal{G}^k(u)$ and $\mathcal{G}^k(v)$, and the expectation is taken with respect to the generating distribution of the graph and the randomness in choosing $u$ and $v$.
\end{proposition}
The proof is deferred to Appendix \ref{appendix:proofs}. Proposition \ref{prop:lower_bound} provides a quantitative characterization of task-irrelevant information of GAE. We discuss possible relaxations of the independence assumptions in Appendix \ref{appendix:proofs}.

According to Proposition~\ref{prop:lower_bound}, the redundancy of GAE scales almost linearly with the size of overlapping subgraphs. To design better self-supervised graph learning methods, we need a principled way to reduce the redundancy while keeping task-relevant information $I(U; T|V)$ almost intact.

\section{Present Work: MaskGAE}
In this section, we present the MaskGAE framework for the MGM pretext task, which is developed by taking inspiration from MLM~\cite{bert} and MIM~\cite{mae}.
As shown in Figure~\ref{fig:framework}, MaskGAE is a simple framework tailored by its \emph{asymmetric} design, where an encoder maps the partially observed graph to a latent representation, followed by two decoders reconstructing the information of the masked structure in terms of edge and node levels.
We empirically show that such an asymmetric encoder-decoder architecture helps GAEs learn generalizable and transferable representations easily.
In what follows, we will give the details of the proposed MaskGAE framework from four aspects: masking strategy, encoder, decoder, and learning objective.

\subsection{Masking Strategy}
The major difference between MaskGAE and traditional GAEs is the tailored MGM task, with masking on the input graph as a crucial operation. According to Proposition~\ref{prop:lower_bound}, the redundancy of two paired subgraphs could be reduced significantly if we mask a certain portion of the edges, thereby avoiding a trivial (large) overlapping subgraph. Furthermore, previous empirical evidence~\cite{dropedge} suggested that edge-level information is often redundant for downstream tasks like node classification, thus implicitly implying little degradation over the task-relevant information $I(U; T|V)$.
Let $\mathcal{G}_\text{mask}=(\mathcal{V}, \mathcal{E}_\text{mask})$ denote the graph masked from $\mathcal{G}$ and $\mathcal{G}_\text{vis}$ the remaining visible graph.
Note that $\mathcal{G}_\text{mask} \cup \mathcal{G}_\text{vis} =\mathcal{G}$. Here we introduce two masking strategies that facilitate the MGM task.

\subsubsection{Edge-wise random masking ($\mathcal{T}_{\text{edge}}$).}
A simple and straightforward way to form a masked graph is to sample a subset of edges $\mathcal{E}_{\text{mask}} \subseteq \mathcal{E}$ following a specific distribution, e.g., Bernoulli distribution:
\begin{equation}
    \mathcal{E}_{\text{mask}} \sim \text{Bernoulli}(p),
\end{equation}
where $p<1$ is the masking ratio for the graph. Such an edge masking strategy is widely used in literature~\cite{grace,graphcl,bgrl}. We denote the edge-wise random masking as $\mathcal{T}_{\text{edge}}$ such that $\mathcal{G}_{\text{mask}}=\mathcal{T}_{\text{edge}}(\mathcal{G})$.

\subsubsection{Path-wise random masking ($\mathcal{T}_{\text{path}}$).}
We also propose a novel structured masking strategy that takes \emph{path} as a basic processing unit during sampling. Informally, a path in a graph is a sequence of edges that joins a sequence of adjacent nodes.
Compared to simple edge-wise masking, path-wise masking breaks the short-range connections between nodes, the model must look elsewhere for evidence to fit the structure that is masked out. Therefore, it can better exploit the structure-dependency patterns and capture the high-order proximity for more meaningful MGM tasks. For path-wise masking, we sample a set of masked edges as follows:
\begin{equation}
    \mathcal{E}_{\text{mask}}\sim \text{RandomWalk}(\mathcal{R}, l_\text{walk}),
\end{equation}
where $\mathcal{R} \subseteq \mathcal{V}$ is a set of root nodes to start random walks and $l_\text{walk}$ is the walk length. Here we sample a subset of nodes from graph as root nodes  $\mathcal{R}$ following the Bernoulli distribution, i.e., $\mathcal{R} \sim \text{Bernoulli}(q)$, where $0<q<1$ is the sample ratio for the graph.
We perform simple random walk~\cite{PerozziAS14} to sample masked
edges $\mathcal{E}_{\text{mask}}$, although biased random walk~\cite{grover2016node2vec} are also applicable. We denote path-wise random masking as $\mathcal{T}_{\text{path}}$. To the best of our knowledge, our method is the first to use path-wise masking for graph self-supervised learning.

\subsubsection{Relation with prior work.}
Both MaskGAE and some existing contrastive methods~\cite{grace,graphcl,bgrl} apply masking on the graph. While those contrastive methods use edge masking as an augmentation to generate different structural views for contrasting, MaskGAE employs edge masking for constructing meaningful supervision signals as well as reducing the redundancy between paired subgraph views, thus facilitating the self-supervised learning scheme.

\subsection{Encoder}
In this work, our encoder is graph convolutional networks (GCN)~\cite{kipf2016semi}, a well-established GNN architecture widely used in literature~\cite{velickovic2019deep,mvgrl}.
Our framework allows various choices of encoder architectures, such as SAGE~\cite{hamilton2017inductive} and GAT~\cite{velickovic2018graph}, without any constraints. We opt for simplicity and adopt the commonly used GCN and discuss different alternatives in Section~\ref{sec:abla}.
Unlike traditional GAEs, our encoder only needs to process a small portion of edges during training, since it is applied only on a visible, unmasked subgraph.
This offers an opportunity for designing an efficient and powerful encoder while alleviating the scalability problem to pretrain large GNNs.

\subsection{Decoder}

\subsubsection{Structure decoder.}
The structure decoder is a basic design of GAEs, which by aggregating pairwise node representations as link representations to decode the graph. There are several ways to design such a decoder in literature, one can use the inner product or a neural network for decoding. We define the structure decoder $h_\omega$ with parameters $\omega$ as:
\begin{equation}
    h_\omega(z_u,z_v)=\text{Sigmoid}(\text{MLP}(z_u\circ z_v)),
\end{equation}
where $\text{MLP}$ denotes a multilayer perceptron and $\circ$ is the element-wise product.

\subsubsection{Degree decoder.}
We also introduce the degree decoder as an auxiliary model to balance the proximity and structure information. As the graph structure itself has abundant supervision signals more than edge connections, we can force the model to approximate the node degree in a masked graph to facilitate the training. We define the degree decoder as:
\begin{equation}
    g_\phi(z_v)=\text{MLP}(z_v),
\end{equation}
with $\phi$ parameterizing the degree decoder.

\subsection{Learning objective}
The loss for MaskGAE has two parts: (i) \textbf{Reconstruction loss}. The reconstruction loss measures how well the model rebuilds the \emph{masked} graph in terms of edge-reconstruction, which in a form is similar to Eq.~\eqref{eq:gae}, but replaces the term $\mathcal{E}^+=\mathcal{E}_\text{mask}$.
(ii) \textbf{Regression loss.} The regression loss measures how closely the prediction of node degree matches the original one in the masked graph. We compute the mean squared error (MSE) between the approximated and original degree in terms of node level:
\begin{equation}
    \label{eq:deg}
    \mathcal{L}_{\text{deg}} = \frac{1}{|\mathcal{V}|}\sum_{v\in \mathcal{V}} ||g_\phi(z_v)-\text{deg}_\text{mask}(v)||_F^2,
\end{equation}
where $\text{deg}_\text{mask}(\cdot)$ denotes the node degree in the masked graph $\mathcal{G}_\text{mask}$. Essentially, $\mathcal{L}_{\text{deg}}(\cdot)$ can work as a regularizer for the encoder to learn more generalizable representations. Thus, the overall objective to be minimized is:
\begin{equation}\label{eq:loss}
    \mathcal{L} = \mathcal{L}_{\text{GAEs}} + \alpha\mathcal{L}_{\text{deg}},
\end{equation}
where $\alpha$ is a non-negative hyperparameter trading off two terms.
The algorithm of MaskGAE is provided in Appendix~\ref{appendix:algo}.

\subsection{Masking reduces subgraph overlapping}

\begin{figure}[t]
    \centering
    \includegraphics[width=\linewidth]{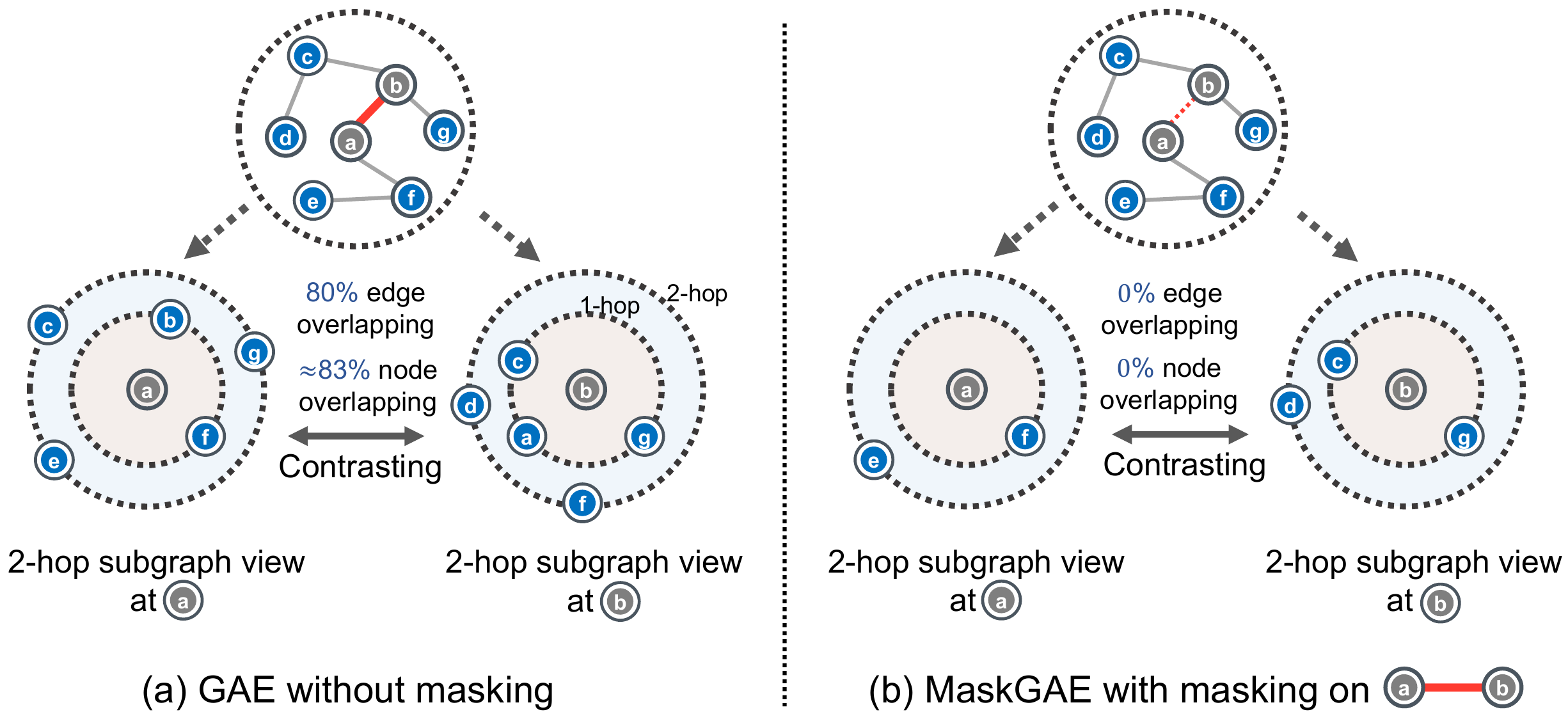}
    \caption{Illustrative examples of the benefits of the masking-and-predicting scheme. Masking on the positive edge helps the contrastive scheme, as it significantly reduces the redundancy of two paired subgraph views.}
    \label{fig:masking}
\end{figure}

\begin{table}[t]
    \centering
    \caption{Statistics of subgraph overlapping (\%).}\label{tab:overlap}
    \resizebox{\linewidth}{!}
    {\begin{tabular}{l c ccc ccc}
            \toprule
             &                           & \multicolumn{3}{c}{\textbf{Cora}} & \multicolumn{3}{c}{\textbf{CiteSeer}}                                                                                                \\
            \cmidrule{3-5}  \cmidrule{6-8}
             &                           & w/o mask                          & $\mathcal{T}_\text{edge}$             & $\mathcal{T}_\text{path}$ & w/o mask & $\mathcal{T}_\text{edge}$ & $\mathcal{T}_\text{path}$ \\
            \midrule
            \multirow{2}{*}{\textbf{$k=2$}}
             & $\mathcal{O}_\text{node}$ & 63.43                             & 31.74                                 & 26.97                     & 71.93    & 33.01                     & 17.05                     \\
             & $\mathcal{O}_\text{edge}$ & 62.17                             & 29.57                                 & 24.90                     & 70.97    & 30.94                     & 15.66                     \\

            \midrule
            \multirow{2}{*}{\textbf{$k=3$}}
             & $\mathcal{O}_\text{node}$ & 69.11                             & 39.66                                 & 33.56                     & 78.90    & 41.07                     & 22.68                     \\
             & $\mathcal{O}_\text{edge}$ & 68.68                             & 38.78                                 & 32.84                     & 78.61    & 40.66                     & 22.34                     \\

            \bottomrule
        \end{tabular}
    }
\end{table}

\paragraph{Illustrative example.}
We first showcase how the masking paradigm assists the self-supervised learning of GAEs. As shown in Figure~\hyperref[fig:masking]{(3a)}, we illustrate how vanilla GAE fails to learn better representations. Consider an edge $a \leftrightarrow b$ to be a positive example, the subgraphs centered at $a$ and $b$ naturally form two contrastive views during message propagation. However, there exists a large overlapping in terms of nodes and edges between the paired subgraph views, which might potentially hinder the contrastive learning of GAE. In contrast, MaskGAE with masking on the edge $a \leftrightarrow b$ can avoid a trivial (large) overlapping subgraph and thus benefit the contrastive scheme of GAEs (see Figure~\hyperref[fig:masking]{3(b)}).

\paragraph{Subgraph overlapping w/ and w/o masking.}
To support our claim, we compute the average overlapping of all paired subgraphs in terms of node and edge levels, denoting as $\mathcal{O}_\text{node}$ and $\mathcal{O}_\text{edge}$, respectively. The computation of both metrics is as follows:
\begin{equation}
    \small
    \begin{aligned}
        \mathcal{O}_\text{node} & = \frac{1}{2|\mathcal{E}^+|}\sum_{(u,v)\in \mathcal{E}^+} \left(\frac{|\mathcal{V}^k(u)\cap \mathcal{V}^k(v)|}{|\mathcal{V}^k(u)|}+\frac{|\mathcal{V}^k(u)\cap \mathcal{V}^k(v)|}{|\mathcal{V}^k(v)|}\right), \\
        \mathcal{O}_\text{edge} & = \frac{1}{2|\mathcal{E}^+|}\sum_{(u,v)\in \mathcal{E}^+} \left(\frac{|\mathcal{E}^k(u)\cap \mathcal{E}^k(v)|}{|\mathcal{E}^k(u)|}+\frac{|\mathcal{E}^k(u)\cap \mathcal{E}^k(v)|}{|\mathcal{E}^k(v)|}\right),
    \end{aligned}
\end{equation}
where $\mathcal{V}^k(v)$ and $\mathcal{E}^k(v)$ represent the node set and edge set w.r.t. the $k$-hop subgraph at node $v$, respectively. Here we compute only on the positive edges, i.e., $\mathcal{E}^+=\mathcal{E}$ (w/o masking) or $\mathcal{E}^+=\mathcal{E}_\text{mask}$ (w/ masking).

\paragraph{Empirical results.}
Table~\ref{tab:overlap} presents the overlapping results on two citation graphs (data statistics are listed in Table~\ref{tab:dataset}) in terms of different hops ($k$) and masking strategies ($\mathcal{T}_\text{edge}$ and $\mathcal{T}_\text{path}$). As shown, both $\mathcal{O}_\text{node}$ and $\mathcal{O}_\text{edge}$ are significantly reduced with our masking strategies. Notably, $\mathcal{T}_\text{path}$ exhibits better capability than $\mathcal{T}_\text{edge}$ in reducing the subgraph overlapping, as evidenced by consistently lower $\mathcal{O}_\text{node}$ and $\mathcal{O}_\text{edge}$ on two datasets. Overall, the result demonstrates the benefits of masking on graphs particularly path-wise masking, which is consistent with experimental results in Section~\ref{sec:exp}.

\section{Experiments}\label{sec:exp}
In this section, we compare MaskGAE with state-of-the-art methods on two fundamental task: link prediction and node classification. The experimental results are also used to validate our theoretical findings presented previously.

\subsection{Experimental setting}

\subsubsection{Datasets.} We conduct experiments on eight well-known benchmark datasets, including three citation networks, i.e., Cora, CiteSeer, Pubmed~\cite{sen2008collective}, two co-purchase graphs, i.e., Photo, Computer~\cite{shchur2018pitfalls}, as well as three large datasets arXiv, MAG and Collab from Open Graph Benchmark~\cite{hu2020ogb}.
The datasets are collected from real-world networks belonging to different domains.
For link prediction task, we randomly split the dataset with 85\%/5\%/10\% of edges for training/validation/testing except for Collab which uses public splits.
All the models are trained on the graph with the training edges.
For the node classification task, we adopt a 1:1:8 train/validation/test random splits for Photo and Computer datasets, as they do not have public splits. The detailed statistics of all datasets are summarized in Table~\ref{tab:dataset}.

\begin{table}[t]
    \centering
    \caption{Dataset statistics.}\label{tab:dataset}
    \begin{threeparttable}
        \resizebox{\linewidth}{!}
        {\begin{tabular}{l|ccccc}
                \toprule
                \multirow{2}{*}{\textbf{Dataset}} & \multirow{2}{*}{\#Nodes} & \multirow{2}{*}{\#Edges} & \multirow{2}{*}{\#Features} & \multirow{2}{*}{\#Classes} & \multirow{2}{*}{Density} \\
                                                  &                          &                          &                             &                            &                          \\
                \midrule
                \textbf{Cora}                     & 2,708                    & 10,556                   & 1,433                       & 7                          & 0.144\%                  \\
                \textbf{CiteSeer}                 & 3,327                    & 9,104                    & 3,703                       & 6                          & 0.082\%                  \\
                \textbf{Pubmed}                   & 19,717                   & 88,648                   & 500                         & 3                          & 0.023\%                  \\
                \textbf{Photo}                    & 7,650                    & 238,162                  & 745                         & 8                          & 0.407\%                  \\
                \textbf{Computer}                 & 13,752                   & 491,722                  & 767                         & 10                         & 0.260\%                  \\
                \textbf{arXiv}                    & 16,9343                  & 2,315,598                & 128                         & 40                         & 0.008\%                  \\
                \textbf{MAG}                      & 736,389                  & 10,792,672               & 128                         & 349                        & 0.002\%                  \\
                \textbf{Collab}                   & 235,868                  & 1,285,465                & 128                         & -                          & 0.002\%                  \\
                \bottomrule
            \end{tabular}
        }
    \end{threeparttable}
\end{table}

\begin{table*}[t]
    \centering
    \caption{Link prediction results (\%) on four graph benchmark datasets. In each column, the \textbf{boldfaced} score denotes the best result and the \underline{underlined} score represents the second-best result.}
    \label{tab:link_predict}
    {\begin{tabular}{lccccccc}
            \toprule
                             & \multicolumn{2}{c}{\textbf{Cora}} & \multicolumn{2}{c}{\textbf{CiteSeer}} & \multicolumn{2}{c}{\textbf{Pubmed}} & \multicolumn{1}{c}{\textbf{Collab}}                                                                                                    \\
            \cmidrule{2-8}
                             & AUC                               & AP                                    & AUC                                 & AP                                  & AUC                            & AP                             & Hit@50                         \\
            \midrule
            GAE              & 91.09 {$\pm$ 0.01}                & 92.83 {$\pm$ 0.03}                    & 90.52 {$\pm$ 0.04}                  & 91.68 {$\pm$ 0.05}                  & 96.40 {$\pm$ 0.01}             & 96.50 {$\pm$ 0.02}             & 47.14 {$\pm$ 1.45}             \\
            VGAE             & 91.40 {$\pm$ 0.01}                & 92.60 {$\pm$ 0.01}                    & 90.80 {$\pm$ 0.02}                  & 92.00 {$\pm$ 0.02}                  & 94.40 {$\pm$ 0.02}             & 94.70 {$\pm$ 0.02}             & 45.53 {$\pm$ 1.87}             \\
            ARGA             & 92.40 {$\pm$ 0.00}                & 93.23 {$\pm$ 0.00}                    & 91.94 {$\pm$ 0.00}                  & 93.03 {$\pm$ 0.00}                  & 96.81 {$\pm$ 0.00}             & 97.11 {$\pm$ 0.00}             & 28.39 {$\pm$ 2.51}             \\
            ARVGA            & 92.40 {$\pm$ 0.00}                & 92.60 {$\pm$ 0.00}                    & 92.40 {$\pm$ 0.00}                  & 93.00 {$\pm$ 0.00}                  & 96.50 {$\pm$ 0.00}             & 96.80 {$\pm$ 0.00}             & 27.32 {$\pm$ 2.93}             \\
            SAGE             & 86.33 {$\pm$ 1.06}                & 88.24 {$\pm$ 0.87}                    & 85.65 {$\pm$ 2.56}                  & 87.90 {$\pm$ 2.54}                  & 89.22 {$\pm$ 0.87}             & 89.44 {$\pm$ 0.82}             & 54.63 {$\pm$ 1.12}             \\
            SEAL             & 92.22 {$\pm$ 1.12}                & 93.12 {$\pm$ 1.01}                    & 93.38 {$\pm$ 0.46}                  & 94.27 {$\pm$ 0.26}                  & 92.99 {$\pm$ 0.99}             & 94.04 {$\pm$ 0.80}             & 64.74 {$\pm$ 0.43}             \\
            MGAE             & 95.05 {$\pm$ 0.76}                & 94.50 {$\pm$ 0.86}                    & 94.85 {$\pm$ 0.49}                  & 94.68 {$\pm$ 0.34}                  & 98.45 {$\pm$ 0.03}             & 98.22 {$\pm$ 0.05}             & 54.74 {$\pm$ 1.06}             \\
            GraphMAE         & 94.88 {$\pm$ 0.23}                & 93.52 {$\pm$ 0.51}                    & 94.32 {$\pm$ 0.40}                  & 93.54 {$\pm$ 0.22}                  & 96.24 {$\pm$ 0.36}             & 95.47 {$\pm$ 0.41}             & 53.97 {$\pm$ 0.64}             \\
            \midrule
            MaskGAE$_{edge}$ & \underline{96.42 {$\pm$ 0.17}}    & \underline{95.91 {$\pm$ 0.25}}        & \textbf{98.02 {$\pm$ 0.22}}         & \textbf{98.18 {$\pm$ 0.21}}         & \underline{98.75 {$\pm$ 0.04}} & \underline{98.66 {$\pm$ 0.06}} & \underline{65.84 {$\pm$ 0.47}} \\
            MaskGAE$_{path}$ & \textbf{96.45 {$\pm$ 0.18}}       & \textbf{95.95 {$\pm$ 0.21}}           & \underline{97.87 {$\pm$ 0.22}}      & \underline{98.09 {$\pm$ 0.17}}      & \textbf{98.84 {$\pm$ 0.04}}    & \textbf{98.78 {$\pm$ 0.05}}    & \textbf{65.98 {$\pm$ 0.39}}    \\
            \bottomrule
        \end{tabular}
    }

\end{table*}

\subsubsection{MaskGAE setup.}
According to different masking strategies, we have MaskGAE$_{edge}$ and MaskGAE$_{path}$, denoting MaskGAE with \textbf{\underline{e}}dge-wise and \textbf{\underline{p}}ath-wise masking strategies, respectively. By default, we fix the mask ratio $p$ to $0.7$ for $\mathcal{T}_\text{edge}$. For $\mathcal{T}_\text{path}$, we randomly sample 70\% of nodes from $\mathcal{V}$ as root nodes $\mathcal{R}$ (i.e., $q=0.7$), performing random walk~\cite{PerozziAS14} starting from $\mathcal{R}$ with $l_\text{walk}=k+1$. For all datasets except arXiv, two GCN layers are applied for the encoder, and two MLP layers are applied for both structure and degree decoders. We increase the number of layers to four for arXiv dataset. Batch normalization~\cite{batchnorm} and ELU~\cite{elu} activation are applied on every hidden layer of encoder. We fix $d_h=64$ for Cora, CiteSeer, and Pubmed, $d_h=128$ for Photo and Computer, $d_h=512$ for arXiv and MAG. For $\alpha$, grid search was performed over the following search space: $\{0, 1e^{-3}, \ldots, 1e^{-2}\}$. In the self-supervised pre-training stage, we employ an Adam optimizer with an initial learning rate of 0.01 and train for 500 epochs except 100 for arXiv and MAG. We also use early stopping with the patience of 30, where we stop training if there is no further improvement on the validation accuracy during 30 epochs.

\subsubsection{Evaluation.}
We consider two fundamental tasks associated with graph representation learning: link prediction and node classification.
For link prediction task, we randomly sample 10\% existing edges from each dataset as well as the same number of nonexistent edges (unconnected node pairs) as testing data. For node classification task, we follow the linear evaluation scheme as introduced in ~\cite{velickovic2019deep,cca_ssg}: We first train the model in a graph with self-defined supervisions and pretext tasks. We then \emph{concatenate} the output of each layer as the node representation. The final evaluation is done by fitting a linear classifier (i.e., a logistic regression model) on top of the frozen learned embeddings without flowing any gradients back to the encoder.
Note that we did not compare with the related work MGAE~\cite{mgae} in the node classification task as the official code is not released and it adopts a rather different experimental setting.

\subsubsection{Software and hardware Specifications.}
Our framework is built upon PyTorch~\cite{pytorch} and PyTorch Geometric~\cite{pyg}. All datasets used throughout experiments are publicly available in PyTorch Geometric library. All experiments are done on a single NVIDIA RTX 2080 Ti GPU (with 11GB memory).

\begin{table*}[t]
    \centering
    \caption{Node classification accuracy (\%) on seven benchmark datasets.  In each column, the \textbf{boldfaced} score denotes the best result and the \underline{underlined} score represents the second-best result. }
    \label{tab:node_clas}
    \begin{tabular}{lccccccc}
        \toprule
                         & \textbf{Cora}                  & \textbf{CiteSeer}              & \textbf{Pubmed}                & \textbf{Photo}                 & \textbf{Computer}              & \textbf{arXiv}                 & \textbf{MAG}                   \\
        \midrule
        MLP              & 47.90 {$\pm$ 0.40}             & 49.30 {$\pm$ 0.30}             & 69.10 {$\pm$ 0.20}             & 78.50 {$\pm$ 0.20}             & 73.80 {$\pm$ 0.10}             & 56.30 {$\pm$ 0.30}             & 22.10 {$\pm$ 0.30}             \\
        GCN              & 81.50 {$\pm$ 0.20}             & 70.30 {$\pm$ 0.40}             & 79.00 {$\pm$ 0.50}             & 92.42 {$\pm$ 0.22}             & 86.51 {$\pm$ 0.54}             & 70.40 {$\pm$ 0.30}             & 30.10 {$\pm$ 0.30}             \\
        GAT              & 83.00 {$\pm$ 0.70}             & 72.50 {$\pm$ 0.70}             & 79.00 {$\pm$ 0.30}             & 92.56 {$\pm$ 0.35}             & 86.93 {$\pm$ 0.29}             & 70.60 {$\pm$ 0.30}             & 30.50 {$\pm$ 0.30}             \\
        \midrule
        GAE              & 74.90 {$\pm$ 0.40}             & 65.60 {$\pm$ 0.50}             & 74.20 {$\pm$ 0.30}             & 91.00 {$\pm$ 0.10}             & 85.10 {$\pm$ 0.40}             & 63.60 {$\pm$ 0.50}             & 27.10 {$\pm$ 0.30}             \\
        VGAE             & 76.30 {$\pm$ 0.20}             & 66.80 {$\pm$ 0.20}             & 75.80 {$\pm$ 0.40}             & 91.50 {$\pm$ 0.20}             & 85.80 {$\pm$ 0.30}             & 64.80 {$\pm$ 0.20}             & 27.90 {$\pm$ 0.20}             \\
        ARGA             & 77.95 {$\pm$ 0.70}             & 64.44 {$\pm$ 1.19}             & 80.44 {$\pm$ 0.74}             & 91.82 {$\pm$ 0.08}             & 85.86 {$\pm$ 0.11}             & 67.34 {$\pm$ 0.09}             & 28.36 {$\pm$ 0.12}             \\
        ARVGA            & 79.50 {$\pm$ 1.01}             & 66.03 {$\pm$ 0.65}             & 81.51 {$\pm$ 1.00}             & 91.51 {$\pm$ 0.09}             & 86.02 {$\pm$ 0.11}             & 67.43 {$\pm$ 0.08}             & 28.32 {$\pm$ 0.18}             \\
        GraphMAE         & \underline{84.20 {$\pm$ 0.40}} & \underline{73.40 {$\pm$ 0.40}} & 81.10 {$\pm$ 0.40}             & 93.23      {$\pm$ 0.13}        & 89.51 {$\pm$ 0.08}             & \textbf{71.75 {$\pm$ 0.17}}    & 32.25 {$\pm$ 0.37}             \\
        \midrule
        DGI              & 82.30 {$\pm$ 0.60}             & 71.80 {$\pm$ 0.70}             & 76.80 {$\pm$ 0.60}             & 91.61 {$\pm$ 0.22}             & 83.95 {$\pm$ 0.47}             & 65.10 {$\pm$ 0.40}             & 31.40 {$\pm$ 0.30}             \\
        GMI              & 83.00 {$\pm$ 0.30}             & 72.40 {$\pm$ 0.10}             & 79.90 {$\pm$ 0.20}             & 90.68 {$\pm$ 0.17}             & 82.21 {$\pm$ 0.31}             & 68.20 {$\pm$ 0.20}             & 29.50 {$\pm$ 0.10}             \\
        GRACE            & 81.90 {$\pm$ 0.40}             & 71.20 {$\pm$ 0.50}             & 80.60 {$\pm$ 0.40}             & 92.15 {$\pm$ 0.24}             & 86.25 {$\pm$ 0.25}             & 68.70 {$\pm$ 0.40}             & 31.50 {$\pm$ 0.30}             \\
        GCA              & 81.80 {$\pm$ 0.20}             & 71.90 {$\pm$ 0.40}             & 81.00 {$\pm$ 0.30}             & 92.53 {$\pm$ 0.16}             & 87.85 {$\pm$ 0.31}             & 68.20 {$\pm$ 0.20}             & 31.40 {$\pm$ 0.30}             \\
        MVGRL            & 82.90 {$\pm$ 0.30}             & 72.60 {$\pm$ 0.40}             & 80.10 {$\pm$ 0.70}             & 91.70 {$\pm$ 0.10}             & 86.90 {$\pm$ 0.10}             & 68.10 {$\pm$ 0.10}             & 31.60 {$\pm$ 0.40}             \\
        BGRL             & 82.86 {$\pm$ 0.49}             & 71.41 {$\pm$ 0.92}             & 82.05 {$\pm$ 0.85}             & 93.17 {$\pm$ 0.30}             & \textbf{90.34 {$\pm$ 0.19}}    & \underline{71.64 {$\pm$ 0.12}} & 31.11 {$\pm$ 0.11}             \\
        SUGRL            & 83.40 {$\pm$ 0.50}             & 73.00 {$\pm$ 0.40}             & 81.90 {$\pm$ 0.30}             & 93.20 {$\pm$ 0.40}             & 88.90 {$\pm$ 0.20}             & 69.30 {$\pm$ 0.20}             & 32.40 {$\pm$ 0.10}             \\
        CCA-SSG          & 83.59 {$\pm$ 0.73}             & 73.36 {$\pm$ 0.72}             & 80.81 {$\pm$ 0.38}             & 93.14 {$\pm$ 0.14}             & 88.74 {$\pm$ 0.28}             & 69.22 {$\pm$ 0.22}             & 27.57 {$\pm$ 0.41}             \\
        \midrule
        MaskGAE$_{edge}$ & 83.77 {$\pm$ 0.33}             & 72.94 {$\pm$ 0.20}             & \underline{82.69 {$\pm$ 0.31}} & \underline{93.30 {$\pm$ 0.04}} & 89.44 {$\pm$ 0.11}             & 70.97 {$\pm$ 0.29}             & \underline{32.75 {$\pm$ 0.43}} \\
        MaskGAE$_{path}$ & \textbf{84.30 {$\pm$ 0.39}}    & \textbf{73.80 {$\pm$ 0.81}}    & \textbf{83.58 {$\pm$ 0.45}}    & \textbf{93.31 {$\pm$ 0.13}}    & \underline{89.54 {$\pm$ 0.06}} & 71.16 {$\pm$ 0.33}             & \textbf{32.79 {$\pm$ 0.32}}    \\
        \bottomrule
    \end{tabular}
\end{table*}

\subsection{Performance comparison on link prediction}

As GAEs, including our MaskGAE, are initially designed for graph reconstruction, the trained model can easily perform link prediction tasks without additional fine-tuning. For fair comparison, we mainly compare MaskGAE with the following methods, including (V)GAE~\cite{kipf2016variational}, AR(V)GA~\cite{PanHLJYZ18_arge}, SAGE~\cite{hamilton2017inductive}, SEAL~\cite{seal}, and MGAE~\cite{mgae}, which are trained in an end-to-end fashion to perform link prediction. Since GraphMAE is not initially designed for link prediction tasks, we instead trained an MLP decoder on the learned representations for GraphMAE to solve the link prediction task.
Following a standard manner of learning-based link prediction~\cite{kipf2016variational,seal}, we conduct experiments on Cora, CiteSeer, Pubmed, and Collab, removing 5\% edges for validation and 10\% edges for test.
All the models are trained on the graph with the remaining 85\% existing edges.

We report AUC score, average precision (AP), and Hits@$k$ which counts the ratio of positive edges ranked at the $k$-th place or above against all
the negative edges. The results of baselines are quoted from~\cite{PanHLJYZ18_arge,mgae,seal}. As shown in Table~\ref{tab:link_predict}, both variants of MaskGAE achieve the leading performance over all compared algorithms regarding all evaluation metrics.
Specifically, both MaskGAE alternatives significantly outperform vanilla GAE on three citation datasets with an increase in AUC and AP scores by 5\% on average, demonstrating the effectiveness of the proposed MaskGAE framework. In addition, MaskGAE outperforms strong baselines and advances new state-of-the-art on the large dataset Collab. We believe the result comes from the benefits of the MGM paradigm. On one hand, MGM helps self-supervised learning of GAEs by reducing the redundancy of contrastive views. On the other hand, MGM naturally eliminates the discrepancy between pre-training and downstream tasks, as the model is required to predict a portion of edges that are ``invisible'' given partially visible structure in both training and inference stages.
In conclusion, MGM is able to benefit the performance of GAEs in modeling the graph structure and outperforming the leading baselines.

\vspace{-3mm}
\subsection{Performance comparison on node classification}
We also perform node classification under the self-supervised learning setting to validate the effectiveness of learned representations. The model is trained on a partially observed graph in a self-supervised manner, following the same setting as the link prediction task.
We compare against baselines belonging to the three categories: (i) \textbf{generative learning methods}, including (V)GAE~\cite{kipf2016variational}, AR(V)GA~\cite{PanHLJYZ18_arge}, and GraphMAE~\cite{graphmae}; (ii) \textbf{contrastive learning methods}, including DGI~\cite{velickovic2019deep}, GMI~\cite{gmi}, GRACE~\cite{grace}, GCA~\cite{gca}, MVGRL~\cite{mvgrl}, BGRL~\cite{bgrl}, SUGRL~\cite{sugrl}, and CCA-SSG~\cite{cca_ssg}. (iii) \textbf{(semi-)supervised methods}, including MLP, GCN~\cite{kipf2016semi} and GAT~\cite{velickovic2018graph}. Both (i) and (ii) are self-supervised methods. We closely follow the linear evaluation scheme as introduced in~\cite{velickovic2019deep,cca_ssg} and report the classification accuracy on all datasets.

The results are presented in Table~\ref{tab:node_clas}.
For self-supervised approaches, we can see that the contrastive methods generally outperform generative methods (GAEs), which verify the effectiveness of contrastive learning. This also highlights the fact that the learned representations by simple GAEs are not task-agnostic, which cannot generalize to other downstream tasks beyond link prediction. In contrast, both variants of MaskGAE outperform or match the best contrastive approaches across all datasets. In addition, MaskGAE also outperforms the state-of-the-art method GraphMAE on six out of seven datasets.
Even compared with the supervised ones, MaskGAE demonstrates superior performance and outperforms all the baselines. The results clearly demonstrate the effectiveness of the proposed MaskGAE framework and validate our claims.

It is observed that MaskGAE exhibits good computational scaling and downstream performance on large datasets arXiv and MAG. This may be due to the fact that only a small portion of unmasked edges are processed by the encoder during training and thus benefiting the scalability. Note that CCA-SSG has a poor performance on arXiv and MAG since it is essentially a dimension-reduction method, where the ideal embedding dimension ought to be smaller than input one~\cite{cca_ssg}.

While both masking strategies assist self-supervised learning for MaskGAE by recovering the masked edges, they are built on different philosophies of structural masking.
We also see that both variants of MaskGAE using masking strategies achieve the leading performance, while MaskGAE$_{path}$ achieves relatively better performance than MaskGAE$_{edge}$ in most cases. It is probably because that path-wise masking enforces a model to predict a set of edges that are locally correlated, thus creating a task that cannot be easily solved by extrapolation from visible graph structure. In this regard, the learned representations are more robust and generalizable, leading to better performance on downstream tasks. This implies that path-wise masking might be a more promising masking strategy for the MGM task.

\subsection{Ablation study}
\label{sec:abla}
In this subsection, we perform ablation studies over the key components of MaskGAE to understand their functionalities thoroughly.

\begin{figure}[t]
    \centering
    \includegraphics[width=0.47\linewidth, height=0.4\hsize]{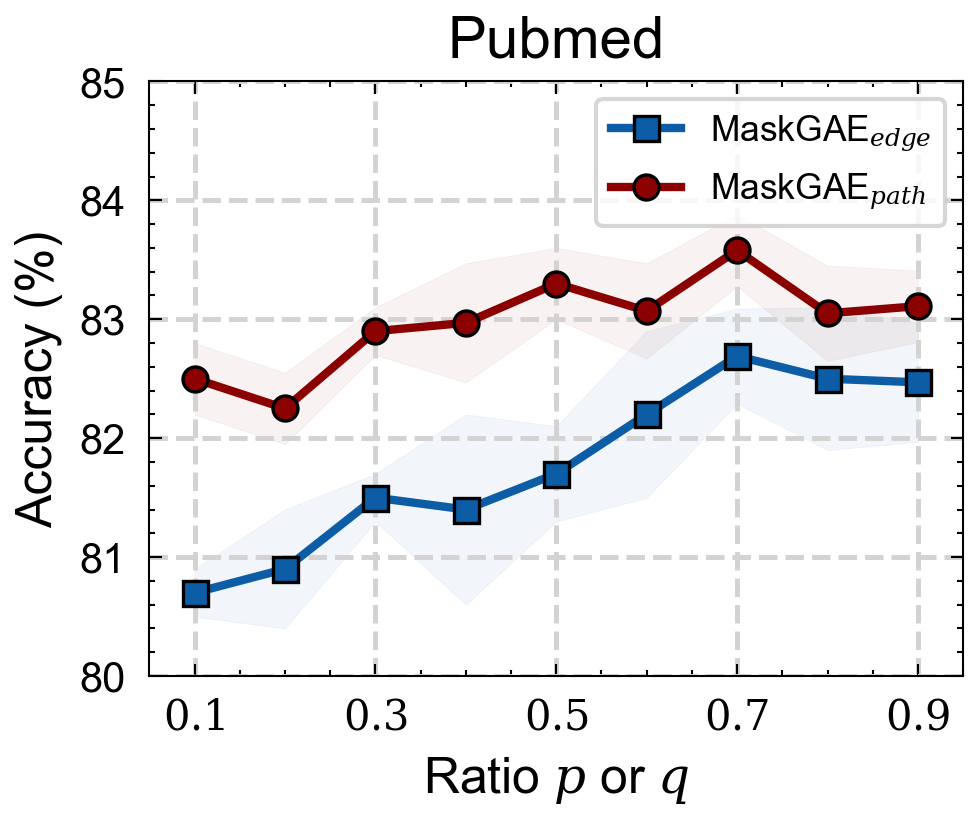}\hspace{0.3cm}
    \includegraphics[width=0.47\linewidth, height=0.4\hsize]{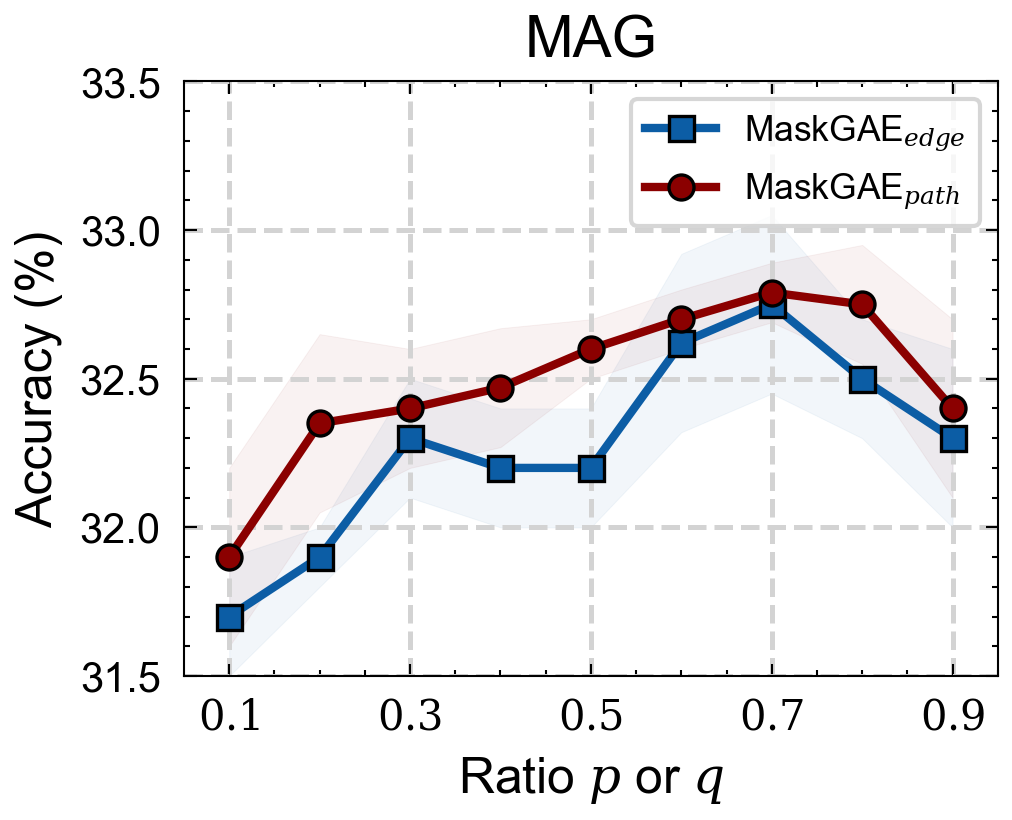}
    \caption{Effect of $p$ for $\mathcal{T}_\text{edge}$ and $q$ for $\mathcal{T}_\text{path}$.}
    \label{fig:ablation_edge}
    \vspace{-3mm}
\end{figure}

\subsubsection{Effect of $p$ and $q$.} Given that masking strategies are critical for our MGM pretext tasks, we first study how the number of masked edges improves or degrades the model performance. Both $p$ and $q$ control the size of masked edges in $\mathcal{T}_\text{edge}$ and $\mathcal{T}_\text{path}$ in different ways. From Figure~\ref{fig:ablation_edge}, we can see that both variants of MaskGAE show similar trends on two datasets when increasing the number of masked edges. The masking strategies indeed yield significant performance improvements on the downstream performance. Particularly, the performance of MaskGAE is smoothly improved when a large mask ratio is adopted, which validates the information redundancy in graphs and also aligns with our theoretical justifications that MGM improves the self-supervised learning scheme. When $p$ or $q$ reaches a critical value, e.g., $0.7$, the best performance is achieved. As shown, two parameters that are too large or too small can lead to poor performance. In this regard, reasonable values of $p$ and $q$ are required to learn useful representations. The results are also in line with previous work~\cite{dropedge} where removing a sufficiently large portion of edges would facilitate the training of GNNs.

\begin{figure}[t]
    \centering
    \includegraphics[width=0.47\linewidth, height=0.4\hsize]{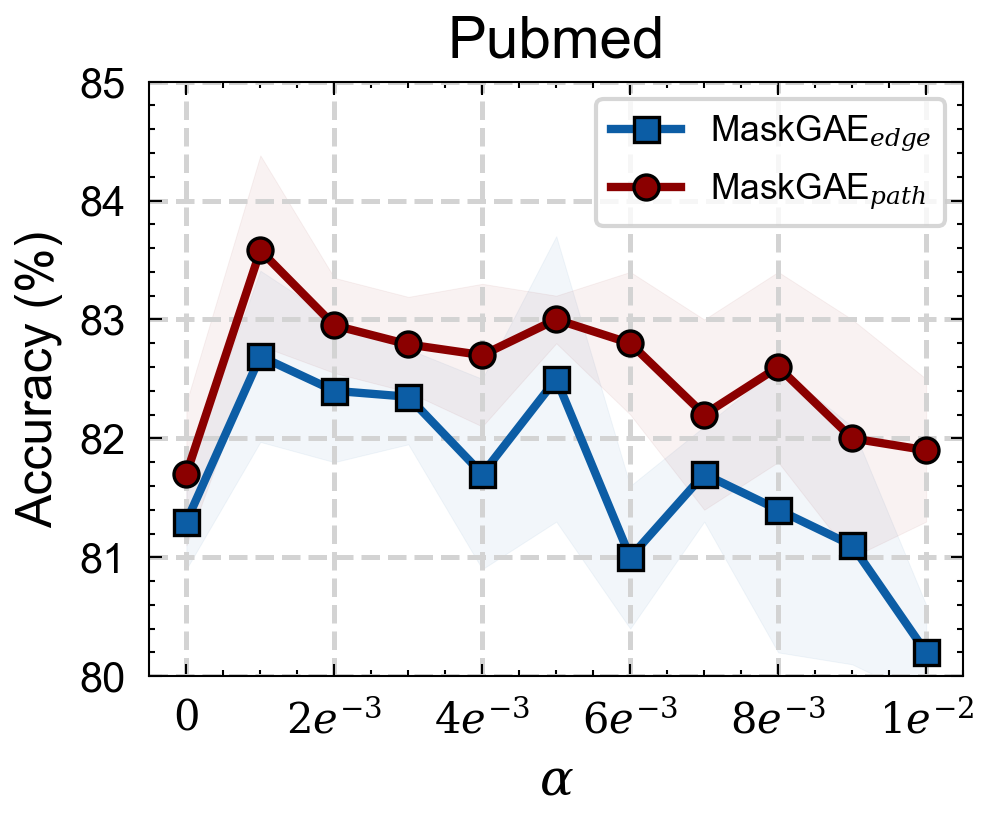}\hspace{0.3cm}
    \includegraphics[width=0.47\linewidth, height=0.4\hsize]{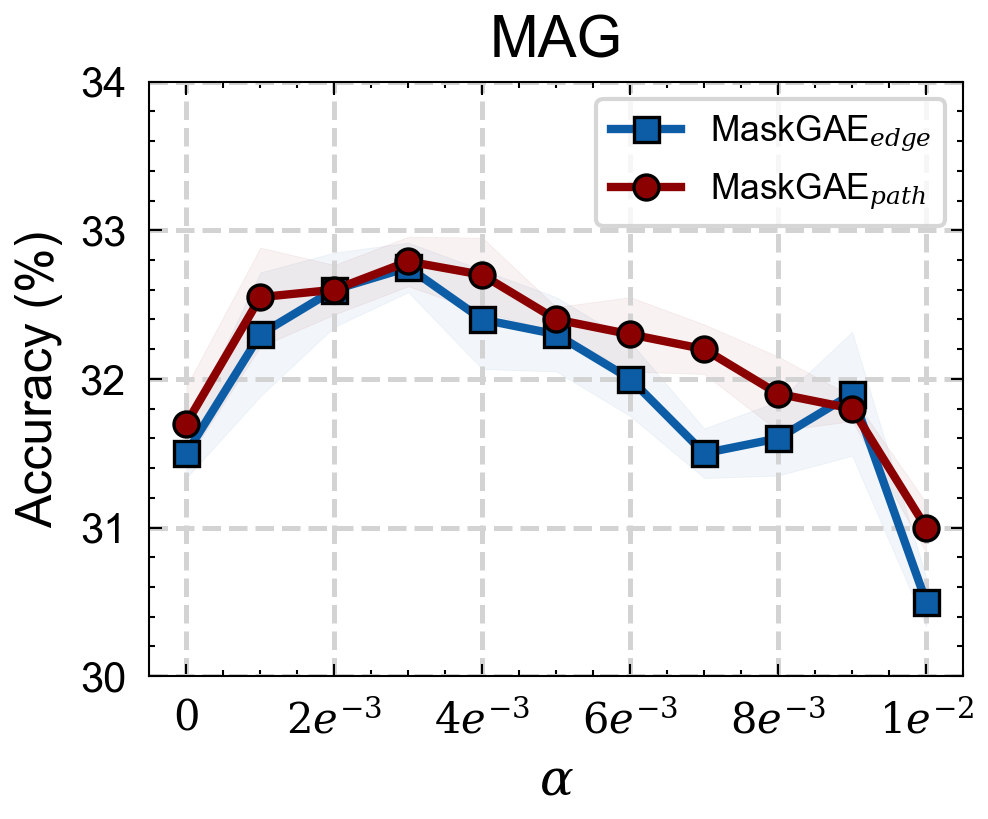}
    \vspace{-5mm}
    \caption{Effect of $\alpha$.}
    \label{fig:ablation_alpha}
    \vspace{-5mm}
\end{figure}

\subsubsection{Effect of $\alpha$.}
We further investigate whether $\mathcal{L}_\text{deg}$ improves the performance as an auxiliary loss by varying the trade-off hyper-parameter $\alpha$ from 0 to $1e^{-2}$. Figure~\ref{fig:ablation_alpha} shows the classification accuracy of both MaskGAE variants w.r.t. different $\alpha$ on Pubmed and MAG. We can see that the performance is boosted when $\alpha>0$ particularly for MaskGAE$_{path}$. This indicates that $\mathcal{L}_\text{deg}$ is important as it contributes to learning good representations for downstream tasks. However, the model would potentially overfit the structure information with a large $\alpha$, as evidenced by increasing $\alpha$ results in a poor performance of MaskGAE.

\subsubsection{Effect of embedding size $d_h$.} Figure~\ref{fig:ablation_embed} ablates the effect of varying embedding size. Embedding size is important for graph representation learning, which reflects the effectiveness of information compression. Compared to contrastive methods~\cite{velickovic2019deep,sugrl,grace}, an interesting finding is that MaskGAE does not benefit significantly from a large embedding dimension and a large range of embedding sizes (64-512) performs equally well, particularly for MaskGAE$_{edge}$.
Compared to sparse datasets Cora and Pubmed, MaskGAE is more stable on dense datasets Photo and Computer.
In addition, we notice a performance drop for MaskGAE as the embedding size increases. Notably, a small embedding size (64 in most cases) is sufficient for MaskGAE, while other contrastive methods generally require much larger dimensions (e.g., 512) to achieve a decent performance~\cite{sugrl,velickovic2019deep}. This indicates that the representations learned by MaskGAE are informative and discriminative for downstream tasks as a small embedding size is required.

\begin{figure}[t]
    \centering
    \includegraphics[width=0.47\linewidth, height=0.4\hsize]{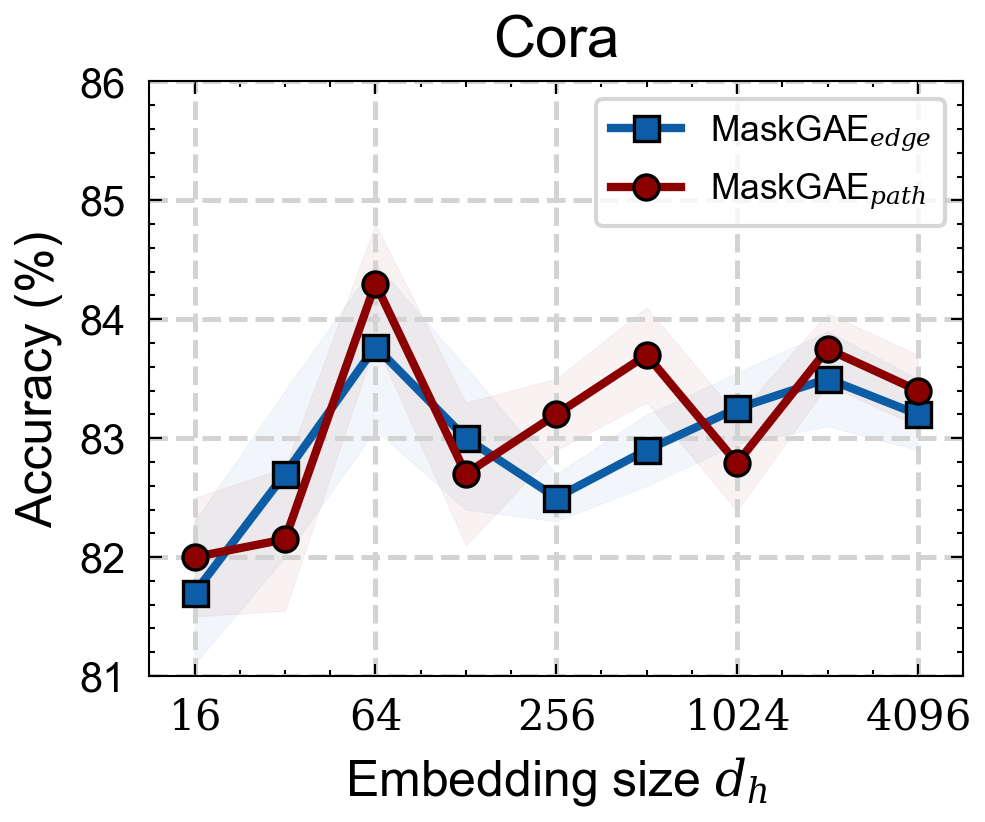}\hspace{0.3cm}
    \includegraphics[width=0.47\linewidth, height=0.4\hsize]{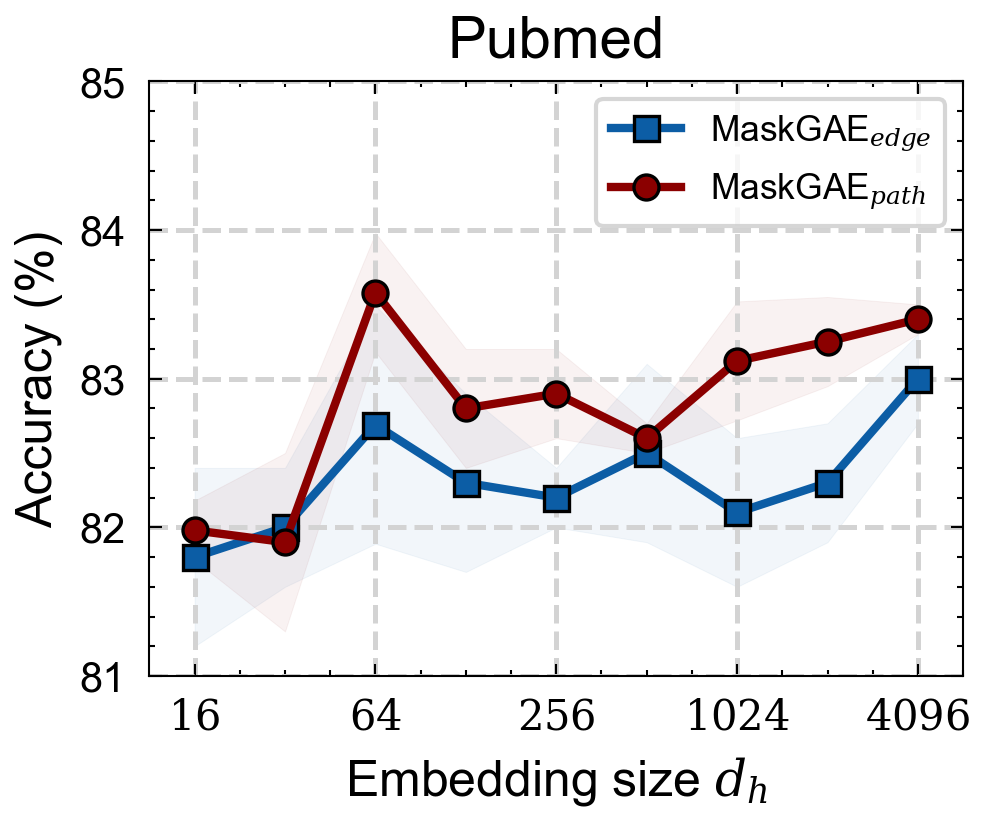}
    \includegraphics[width=0.47\linewidth, height=0.4\hsize]{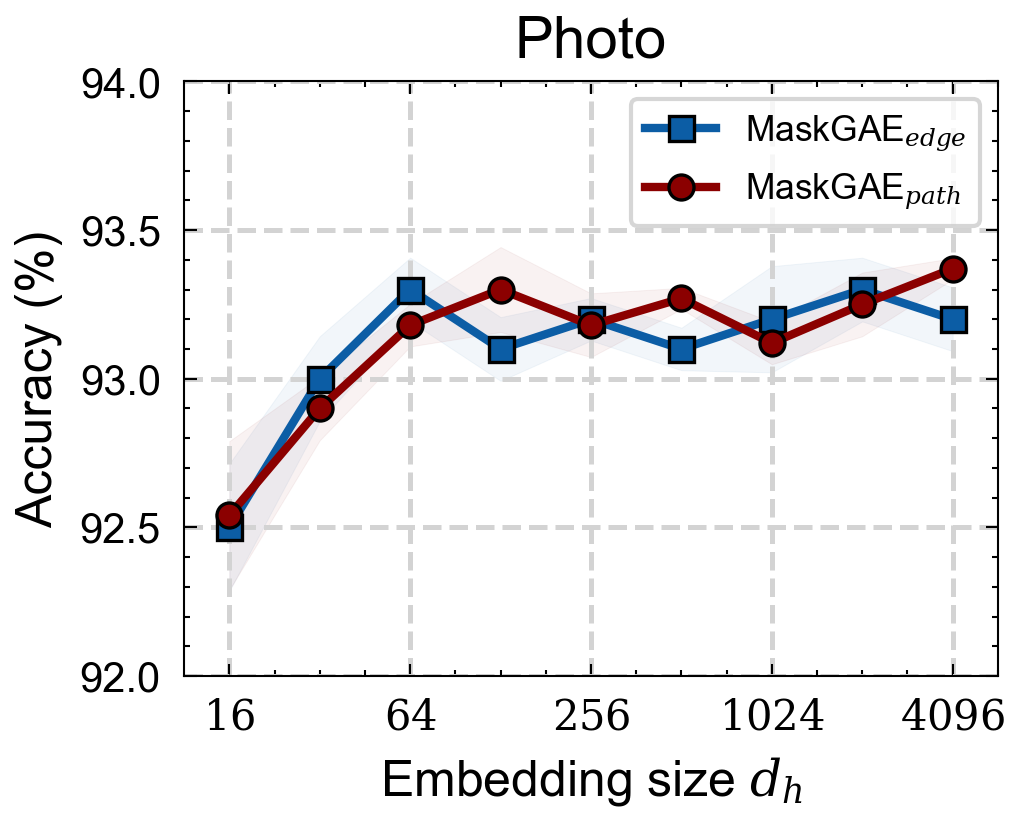}\hspace{0.3cm}
    \includegraphics[width=0.47\linewidth, height=0.4\hsize]{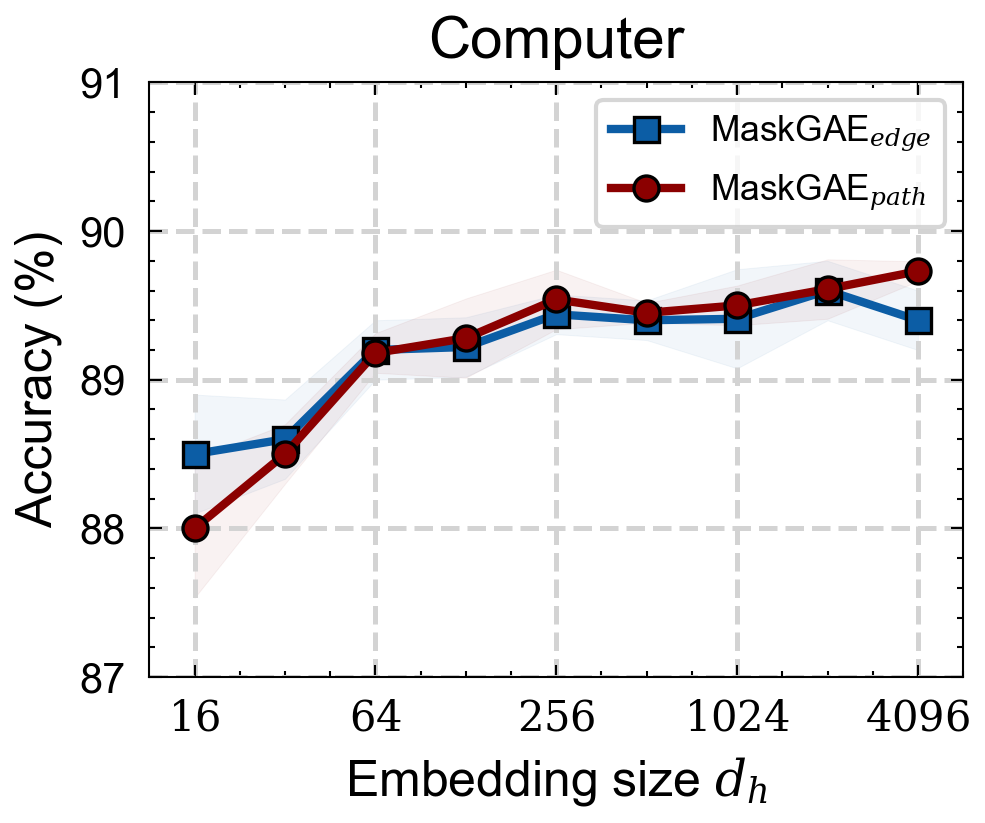}
    \caption{Effect of embedding size $d_h$.}
    \label{fig:ablation_embed}
\end{figure}

\subsubsection{Effect of encoder architecture.}
The encoder plays an important role in mapping graphs into low-dimensional representations. Typically, a wide range of literature~\cite{velickovic2019deep,bgrl,gca,cca_ssg} uses GCN as the default option to implement the encoder and achieve outstanding performance on the node classification task.
To explore the possibility of designing powerful self-supervised methods, we conduct ablation studies on three citation graphs with different GNN encoders, including SAGE~\cite{hamilton2017inductive}, GAT~\cite{velickovic2018graph}, and GCN~\cite{kipf2016semi}. The results are shown in Table~\ref{tab:encoder}. We can observe that GCN is the best encoder architecture for MaskGAE with both masking strategies. MaskGAE with GCN as the encoder
exhibits significantly improved performances over GAT and SAGE in all cases, indicating that a simple encoder with the aid of MGM is sufficient to learn useful representations.

\begin{table}[t]
    \centering
    \caption{Accuracy (\%) of MaskGAE with different encoders in node classification task.}\label{tab:encoder}
    \resizebox{\linewidth}{!}
    {\begin{tabular}{lcccc}
            \toprule
             &      & \textbf{Cora}               & \textbf{CiteSeer}           & \textbf{Pubmed}             \\
            \midrule
            \multirow{3}{*}{\textbf{MaskGAE$_{edge}$}}
             & SAGE & 81.32 {$\pm$ 0.18}          & 71.19 {$\pm$ 0.10}          & 80.17 {$\pm$ 0.36}          \\
             & GAT  & 81.99 {$\pm$ 0.14}          & 71.95 {$\pm$ 0.35}          & 81.21 {$\pm$ 0.12}          \\
             & GCN  & \textbf{83.77 {$\pm$ 0.33}} & \textbf{72.94 {$\pm$ 0.20}} & \textbf{82.69 {$\pm$ 0.31}} \\
            \midrule
            \multirow{3}{*}{\textbf{MaskGAE$_{path}$}}
             & SAGE & 80.66 {$\pm$ 0.12}          & 72.49 {$\pm$ 0.17}          & 79.73 {$\pm$ 0.28}          \\
             & GAT  & 82.18 {$\pm$ 0.23}          & 72.99 {$\pm$ 0.30}          & 81.84 {$\pm$ 0.46}          \\
             & GCN  & \textbf{84.30 {$\pm$ 0.39}} & \textbf{73.80 {$\pm$ 0.81}} & \textbf{83.58 {$\pm$ 0.45}} \\
            \bottomrule
        \end{tabular}
    }
\end{table}

\section{Conclusion}
In this work, we make a comprehensive investigation on masked graph modeling (MGM) and present MaskGAE, a theoretically grounded self-supervised learning framework that takes MGM as a principled pretext task. Our work is conceptually well-understood and theoretically motivated with the following justifications: (i) GAEs are essentially contrastive learning models that maximize the mutual information between paired subgraph views associated with a linked edge, and (ii) MGM can benefit the mutual information maximization since masking significantly reduce the overlapping (redundancy) between two subgraph views. In particular, we also propose a path-wise masking strategy to facilitate the MGM task. In our experiments, MaskGAE exhibits significantly improved performances over GAEs and performs on par with or better than strong baselines on link prediction and node classification benchmarks.

\section{Acknowledgements}
The research is supported by the Guangdong Basic and Applied Basic Research Foundation (2020A1515010831), the Guangzhou Basic and Applied Basic Research Foundation
(202102020881), and CCF-AFSG Research Fund (20210002).

\bibliographystyle{ACM-Reference-Format}
\bibliography{main}

\clearpage
\appendix

\section{Proofs}\label{appendix:proofs}

The key ingredient of the proof of Proposition~\ref{prop:lower_bound} is the following lemma, which is a conditional version of \cite[Lemma 1]{xu2017information}:
\begin{lemma}\label{lem:mi}
    Consider three random variables $X$, $Y$ and $T$. Denote $P_{XY|T}$ as the joint distribution of $X$ and $Y$ given $T$, and $P_{X|T}$ and $P_{Y|T}$ as conditional distribution of $X$ and $Y$ given $T$, respectively. Let $q: \mathcal{X} \times \mathcal{Y} \times \mathcal{T} \mapsto \mathbb{R}$ be a function such that for any given $t \in \mathcal{T}$, $q$ is $\sigma$-subgaussian in the product of distributions $P_{X|T=t} \times P_{Y|T=t}$, i.e.,
    \begin{equation}\label{eq:subgaussian}
        \begin{aligned}
             & \log \mathbb{E}_{x \sim P_{X|T=t}, y \sim P_{Y|T=t}}\left[e^{\lambda (q(x, y, t))}\right] - \\ & \lambda \mathbb{E}_{x \sim P_{X|T=t}, y \sim P_{Y|T=t}}\left[q(x, y, t)\right] \le \dfrac{\lambda^2 \sigma^2}{2}, \forall \lambda.
        \end{aligned}
    \end{equation}
    Then we have the following lower bound on the conditional MI $I(X;Y|T)$:
    \begin{equation}
        \begin{aligned}
            \sigma \sqrt{2 I(X; Y|T)} \ge & | \mathbb{E}_T \mathbb{E}_{x, y \sim P_{XY|T=t}}\left[q(x, y, t)\right] \\ & - \mathbb{E}_{\bar{x}\sim P_{X|T=t}, \bar{y}\sim P_{Y|T=t}}\left[q(\bar{x}, \bar{y}, t)\right]|.
        \end{aligned}
    \end{equation}
\end{lemma}

\begin{proof}[Proof of Proposition~\ref{prop:lower_bound}]
    We show the proposition with the feature dimension $d=1$, and extending to general finite feature dimensions is straightforward. By the conditional version of data processing inequality \cite{polyanskiy2014lecture}, for any \textbf{subgraph encoding map} $\mathcal{M}_k$ that maps  $k$-hop subgraphs of any node $v$ into real vectors, we have:
    \begin{align}
        I(U;V | T) \ge I(\mathcal{M}_k(\mathcal{G}^k(u)); \mathcal{M}_k(\mathcal{G}^k(v)) | T).
    \end{align}
    For notational simplicity we let $m^k(v) = \mathcal{M}_k(\mathcal{G}^k(v))$ for $\forall v \in \mathcal{V}$. Now it suffices to bound $I(U;V | T)$ under a certain encoding map, which we construct as follows: Conditional on the topological information of the underlying graph $\mathcal{G}$ with size $N = |\mathcal{G}|$, we associate each node $v$ with an index $\text{id}(v)$ that corresponds to its order in the topological traversal of the graph. Then $m^k(v)$ is a sparse vector of length $N$, with nonzero positions defined as:
    \begin{align}
        m^k(v)[\text{id}(u)] = x(u), \forall u \in \mathcal{G}^k(v),
    \end{align}
    where we use $x[k]$ to denote the $k$-th element of a vector $x$. To bound $I(m^k(U);m^k(V)|T)$ from below, we use lemma \ref{lem:mi} via choosing $q$ to be the euclidean inner product:
    \begin{equation}\label{eq:inner_product}
        \begin{aligned}
            q(m^k(u), m^k(v), t) & = \langle m^k(u), m^k(v)\rangle    \\
                                 & = \sum_{i=1}^N m^k(u)[i] m^k(v)[i]
        \end{aligned}
    \end{equation}
    First, we assess the subgaussian property of $q$. Note that each summand in Eq.\eqref{eq:inner_product} is a product of two independent zero-mean random variables with range $[-1, 1]$, it follows from standard results of subgaussian random variables \cite{wainwright2019high} that each summand is subgaussian with coefficient $1$. Since Eq.\eqref{eq:inner_product} has at most $N_k$ nonzero summands. It follows by the property of subgaussian random variables implies that $q$ is subgaussian with $\sigma \le \sqrt{N_k}, \forall t \in \mathcal{T}$.
    Next, we investigate the expectation of $q$ under different distributions. Note that given the topological structure, the randomness is only with respect to the generating distribution of node features. It follows from the independence of node features and the zero-mean assumption, that
    \begin{align}
        \mathbb{E}_{\bar{x}\sim P_{X|T=t}, \bar{y}\sim P_{Y|T=t}}q(\bar{x}, \bar{y}, t) = 0, \forall t \in \mathcal{T}.
    \end{align}
    Similarly, we have
    \begin{align}
        \mathbb{E}_{x, y \sim P_{XY|T}}f(x, y, t) = \gamma N^k_{uv}, \forall t \in \mathcal{T},
    \end{align}
    since the rest $N - N^k_{uv}$ pairs are orthogonal. Applying lemma \ref{lem:mi} we get the desired result:
    \begin{align}\label{eqn: infobound}
        I(U;V| T) \ge I(m^k(U); m^k(V) | T) \ge \frac{\left(\mathbb{E}[N^k_{uv}]\right)^2}{N_k}\gamma^2
    \end{align}
\end{proof}

\begin{proof}[Proof of lemma \ref{lem:mi}]
    The proof goes along similar arguments in \cite{xu2017information}, with slight modifications in the conditional case. First by the Donsker-Varadhan variational representation of conditional mutual information:
    \begin{equation}
        \begin{aligned}
            I(X; Y|T) \ge & \mathbb{E}_T [\mathbb{E}_{x, y \sim XY|T=t}\lambda q(x, y, t) \\ & - \log \mathbb{E}_{\bar{x} \sim P_{X|T=t}, \bar{y} \sim P_{Y|T=t}}\left(e^{\lambda q(\bar{x}, \bar{y}, t)}\right)]
        \end{aligned}
    \end{equation}
    By the required subgaussian assumption we have
    \begin{equation}\label{eq:parabola}
        \begin{aligned}
            I(X; Y|T) \ge & \mathbb{E}_T [\lambda\mathbb{E}_{x, y \sim XY|T=t}q(x, y, t) \\ & -  \lambda \mathbb{E}_{\bar{x} \sim P_{X|T=t}, \bar{y} \sim P_{Y|T=t}} q(\bar{x}, \bar{y}, t)] - \dfrac{\lambda^2\sigma^2}{2}
        \end{aligned}
    \end{equation}
    Viewing the equation \eqref{eq:parabola} as a non-negative parabola and setting its discriminant to be nonpositive yield the result.
\end{proof}

\subsection{On Approximating $I(U;V)$}
Recall that $z_v = f_\theta (\mathcal{G})[v]$, for clarification of parameter dependence, we write $z_v := f_\theta(v)$ without further understandings. Minimization of the GAE loss Eq.\eqref{eq:gae} is therefore stated as:
\begin{align}\label{eq:prob}
    \omega^*, \theta^* \in \underset{\omega \in \Omega, \theta \in \Theta}{\arg\min} \mathcal{L}_{\text{GAEs}}(\omega, \theta)
\end{align}
First we consider the ideal situation that the parameter space $\Omega \times \Theta$ is sufficiently large so that there exists $\theta_0 \in \Theta$ and $\omega_0 \in \Omega$ such that
\begin{align}\label{eq:exact_approx}
    h_{\omega_0}\left(f_{\theta_0}(u), f_{\theta_0}(v)\right) = \log \dfrac{p(u, v)}{p(u)p(v)}, \forall u, v
\end{align}
Now we use the fact that the pointwise dependency $\log \dfrac{p(u, v)}{p(u)p(v)}$ \cite{tsai2020neural} is both the optimizer of the Donsker-Varadhan variational objective \cite{polyanskiy2014lecture} and the objective Eq.\eqref{eq:gae_pop} \cite{tsai2020neural}. Standard results on M-estimation \cite[Theorem 5.7]{van2000asymptotic} suggest that, if we are allowed to sample from the generating distribution of $\mathcal{G}$ and get an arbitrary number of samples (we assume the satisfaction of empirical process conditions as in the discussion of \cite{pmlr-v80-belghazi18a}), we have $w^* \overset{p}{\rightarrow}\omega_0$ and $\theta^* \overset{p}{\rightarrow}\theta_0$ with the number of samples goes to infinity. In practice, the actual parameterization may not achieve the ideal property \eqref{eq:exact_approx} using GNN as encoders: standard results in the expressivity of locally unordered message passing \cite{garg2020generalization} indicates that there exist certain pairs of subgraphs that cannot be distinguished by most of the off-shelf GNN models. Consequently the condition \eqref{eq:exact_approx} is violated. Nonetheless, we may view the problem \eqref{eq:prob} as being composed of two subproblems:
\begin{align}
    \omega^* = \omega^*(\theta^*) & \in \underset{\omega \in \Omega}{\arg\min} \mathcal{L}_{\text{GAEs}}(\omega, \theta^*),                       \\
    \theta^*                      & \in \underset{\theta \in \Theta}{\arg\min} \min_{\omega \in \Omega}\mathcal{L}_{\text{GAEs}}(\omega, \theta),
\end{align}
with the inner problem satisfies that for any given $\theta \in \Theta$,
\begin{align}
    \mathcal{I}_{h*(\theta)}(f_\theta(U); f_\theta(V)) = I(f_\theta(U); f_\theta(V)).
\end{align}

Here the function approximation is possible if $\omega$ is parameterized as an MLP with sufficient capacity.

\subsection{Discussions on Proposition \ref{prop:lower_bound}}
\label{appendix:discussions}
There appear some issues regarding the statement of Proposition~\ref{prop:lower_bound}:

\textbf{On the downstream target $T$.} In the original statement, $T$ is set to be the topological information of the underlying graph. This appears to be somewhat necessary since graph learning shall exploit topological structure. A possible relaxation could be to let $T$ be \emph{local topological information} considering only $\mathcal{G}^k(v)$ for any $v \in \mathcal{V}$. Such relaxation does not affect the result of Proposition~\ref{prop:lower_bound} much since we just need to restrict the analysis to certain subgraphs. The situation becomes more complicated if we allow $T$ to be a hybrid of feature and topological information, i.e., there is some additional randomness not captured by $T$ regarding both feature and topology. We left corresponding explorations to future research.

\textbf{On the independence between $X$ and $T$.} The independence assumption might not be appropriate for certain types of generating mechanisms. For example, in many popular probabilistic models of graph-like random graph models \cite{goldenberg2010survey} and Markov random fields \cite{wainwright2008graphical}, if we treat node features as random samples from such graph models, conditional on the graph structure, node features are correlated according to whether a path exists between the underlying nodes. Incorporating dependence in proposition \ref{prop:lower_bound} will result in trivial bounds in the worst case: via investigating the proof, since sums of correlated subgaussian random variables are subgaussian with a potentially much larger coefficient ($N_k$ instead of $\sqrt{N_k}$), resulting in a lower bound no greater than the variance of a single node. This trivial lower bound, however, cannot be improved under general dependence of node features. Consider the following extreme case where the underlying graph is complete, and adjacent nodes are perfectly correlated, i.e., all node features are the same with probability one. In this extreme case, the randomness of the entire graph boils down to the randomness of any single node in the graph, and the mutual information of any two subgraphs will thus not exceeds the entropy of a single node, with is proportional to the variance of the node if the node feature is a zero-mean random variable. A more careful examination of $I(U;V|T)$ under some specified dependence structure between nodes is therefore valuable, and we left it to future studies.

\textbf{On the negative edge sampling procedure.}
To strictly adhere to the infomax paradigm, we need the negative sample pairs to be independently drawn from the marginals. The original GAE objective Eq.\eqref{eq:gae} is therefore a biased estimate of the population objective Eq.\eqref{eq:gae_pop} since the negative samples are drawn from unconnected node pairs. However, the estimation bias does not affect the analysis: denote the true negative sampling distributions as $P^-_{UV}$, then following the procedure in \cite{tsai2020neural}, we can show the following optimizer:
\begin{equation}
    \begin{aligned}
        \widetilde{h}_* \in \underset{h \in \mathcal{H}}{\arg\min} \  & [- \mathbb{E}_{u, v \sim P_{UV}}\log h(u, v) \\ & - \mathbb{E}_{u^\prime, v^\prime \sim P^-_{UV}}\log(1-h(u, v))]
    \end{aligned}
\end{equation}
satisfies that
\begin{equation}
    \begin{aligned}
        D_{\text{KL}}(P_{UV} || P^-_{UV}) & = \mathbb{E}_{u, v \sim P_{UV}} \widetilde{h}_*(u, v) \\ & - \log \mathbb{E}_{u, v \sim P^-_{UV}}(e^{\widetilde{h}_*(u, v)}).
    \end{aligned}
\end{equation}
Since the Kullback-Leibler divergence enjoys the chain rule \cite{polyanskiy2014lecture}, our discussion on GAE follows the modeling paradigm under the true negative sampling process as well, with some modification over the definition of MI.

\section{Algorithm}\label{appendix:algo}
To help better understand the proposed framework, we provide the detailed algorithm and PyTorch-style pseudocode for training MaskGAE in Algorithm~\ref{algo:maskgae} and Algorithm~\ref{algo:pseudo_maskgae}, respectively.

\begin{algorithm}[h]
    \caption{Masked Graph Autoencoder (MaskGAE)}
    \label{algo:maskgae}
    \begin{algorithmic}[0]
        \Require Graph $\mathcal{G}=(\mathcal{V}, \mathcal{E})$, encoder $f_\theta(\cdot)$, structure decoder $h_\omega (\cdot)$, degree decoder $g_\phi(\cdot)$, masking strategy $\mathcal{T} \in \{\mathcal{T}_\text{edge}, \mathcal{T}_\text{path}\}$, hyperparameter $\alpha$;
        \Ensure Learned encoder $f_\theta(\cdot)$;
    \end{algorithmic}
    \begin{algorithmic}[1]
        \While{\emph{not converged}};
        \State $\mathcal{G}_{\text{mask}} \leftarrow \mathcal{T}(\mathcal{G})$;
        \State $\mathcal{G}_{\text{vis}} \leftarrow \mathcal{G}-\mathcal{G}_{\text{mask}}$;
        \State $\mathbf{Z} \leftarrow f(\mathcal{G}_{\text{vis}})$;
        \State Calculate $\mathcal{L}_\text{GAEs}$ over $\mathcal{G}_{\text{mask}}$ according to Eq.\eqref{eq:gae};
        \State Calculate $\mathcal{L}_{\text{deg}}$ over $\mathcal{G}_{\text{mask}}$ according to Eq.\eqref{eq:deg};
        \State Calculate loss function $\mathcal{L}\leftarrow \mathcal{L}_{\text{GAEs}} + \alpha\mathcal{L}_{\text{deg}}$;
        \State Update $\theta, \omega, \phi$ by gradient descent;
        \EndWhile;\\
        \Return $f_\theta(\cdot)$;
    \end{algorithmic}
\end{algorithm}

\begin{algorithm}[H]
    \caption{PyTorch-style pseudocode for training MaskGAE}
    \label{algo:pseudo_maskgae}
    \begin{algorithmic}[0]
        \State \PyComment{f:\ encoder network}
        \State \PyComment{h:\ structure decoder network}
        \State \PyComment{g:\ degree decoder network}
        \State \PyComment{alpha:\ trade-off}
        \State \PyComment{edge\_index:\ visible edges of input graph}
        \State \PyComment{x:\ node features}
        \State
        \State \PyCode{masked\_edges, vis\_edges = \textcolor{pink}{masking}(edge\_index)}
        \State \PyCode{num\_edges = masked\_edges.size(1)}
        \State \PyCode{z = f(x, vis\_edges) \PyComment{In vanilla GAE, z = f(x, edge\_index)}}
        \State \PyCode{neg\_edges = \textcolor{pink}{negative\_sampling}(num\_samples=num\_edges)}
        \State \PyComment{Calculate reconstruction loss}
        \State \PyCode{loss\_pos = F.\textcolor{pink}{crossentropy}(h(z, masked\_edges), torch.ones(num\_edges))}
        \State \PyCode{loss\_neg = F.\textcolor{pink}{crossentropy}(h(z, neg\_edges),   torch.zeros(num\_edges))}
        \State \PyCode{loss\_gaes = loss\_pos + loss\_neg}
        \State \PyComment{Calculate regression loss as a regularizer}
        \State \PyCode{loss\_deg = F.\textcolor{pink}{mse\_loss}(g(z), \textcolor{pink}{degree}(masked\_edges))}
        \State \PyComment{Total loss for optimization}
        \State \PyCode{loss = loss\_gaes + alpha * loss\_deg}
        \State \PyComment{Backward propagation}
        \State \PyCode{loss.\textcolor{pink}{backward}()}
    \end{algorithmic}
\end{algorithm}

\section{Further Discussions}

\begin{table*}[t]
    \centering
    \caption{Comparison of MaskGAE and GraphMAE.}\label{tab:comparison}
    {\begin{tabular}{l|cccc}
            \toprule
            \textbf{Method}          & \textbf{Reconstruct feature} & \textbf{Reconstruct edge} & \textbf{Reconstruct degree} & \textbf{Theoretical contribution} \\
            \midrule
            GraphMAE~\cite{graphmae} & \cmark                       & \xmark                    & \xmark                      & \xmark                            \\
            MaskGAE (ours)           & \xmark                       & \cmark                    & \cmark                      & \cmark                            \\
            \bottomrule
        \end{tabular}
    }
\end{table*}

\paragraph{MaskGAE versus GraphMAE}
GraphMAE and MaskGAE share similar intuitions on the self-supervised learning scheme. Both of them are generative methods and adopt the masking strategy on GAE to improve its performance. However, they are technically different in some aspects (see also Table~\ref{tab:comparison}): (i) GraphMAE focuses on \emph{feature reconstruction} while MaskGAE focuses on \emph{structure reconstruction}, including edge and degree reconstructions. (ii) GraphMAE only presents empirical results, without in-depth analysis or theoretical justification on the benefits of the masked autoencoding scheme on graphs. By contrast, our work digs into the hidden reasoning behind graph autoencoders and graph contrastive learning, and offers insights on \emph{why masking could benefit GAEs}. The theoretical analysis between GAEs and contrastive learning is our main contribution and is non-trivial.

\paragraph{Limitations of the work.} Despite the theoretical grounds and the promising experimental justifications,
our work might potentially suffer from some limitations: (i) As like existing augmentation-based contrastive methods, performing masking (a kind of augmentation) on the graph structure would hurt the semantic meaning of some graphs, such as biochemical molecules~\cite{rong2020self,DBLP:conf/ijcai/ZhaoLHLZ21}. (ii) MaskGAE is mainly based on the homophily assumption, a basic assumption in the literature of graph-based representation learning. However, such an assumption may not always hold in heterophilic graphs, where the labels of linked nodes are likely to differ.

\paragraph{Potential negative societal impact.}
Self-supervised learning on graphs is an important technique to make use of rich unlabeled data and can be applied to many fields such as financial networks~\cite{wang2021review} and molecular biology~\cite{rong2020self,DBLP:conf/ijcai/ZhaoLHLZ21}. Our work presents a self-supervised framework for learning effective and generalizable node representations. These representations enable learning robust classifiers with limited data and thus facilitate applications in many domains where annotations are expensive or difficult to collect. Just like any machine learning algorithm can be used for good it can also be used for harm. For example, an immediate application of this work is to conveniently construct a pretrained dataset by randomly crawling data from elsewhere. This is economic and beneficial for research purposes, but would potentially risk privacy and license issues in many security-sensitive situations. This would be the potential concerns and negative societal impact of our work. The line of self-supervised learning works, including this paper, are not immune to such misuse. Currently, we have no solution for such a general problem but we are aware that this needs to be addressed in future work.

\end{document}